\documentclass[11pt]{article}
\usepackage[margin=1in]{geometry}
\usepackage[authoryear, sort&compress]{natbib}
\usepackage{authblk}

\usepackage{url} %this package should fix any errors with URLs in refs.
\usepackage{lineno}
\usepackage{soul}
\usepackage{microtype}
\usepackage[utf8]{inputenc} % allow utf-8 input
\usepackage[T1]{fontenc}    % use 8-bit T1 fonts
\usepackage{url}            % simple URL typesetting
\usepackage{booktabs}       % professional-quality tables
\usepackage{multirow}
\usepackage{nicefrac}       % compact symbols for 1/2, etc.
\usepackage{tikz}
\usepackage{svg}

\usepackage{array}
\usepackage{graphicx}
\usepackage{paralist}
\usepackage{float}
\usepackage[stable]{footmisc}
\usepackage{tikz}
\usepackage{tikz-cd}
\usetikzlibrary{arrows.meta}
\tikzset{>={Stealth}}
\usepackage{enumitem}
\usepackage{amsmath}
\usepackage{amssymb}
\usepackage{mathtools}
\usepackage{amsthm}
\usepackage{algorithm}
\usepackage{algorithmic}
\usepackage{amsfonts}       % blackboard math symbols
\usepackage{relsize} % table font size adjustments
\usepackage[most]{tcolorbox}
\usepackage{makecell}
\usepackage[hidelinks]{hyperref}
\usepackage{orcidlink}
\providecommand{\doi}[1]{\href{https://doi.org/#1}{doi:\nolinkurl{#1}}}

\theoremstyle{plain}
\newtheorem{theorem}{Theorem}[section]

\theoremstyle{definition}

\theoremstyle{remark}

\newcommand{\diff}{\mathrm{d}}

\newcommand{\Z}{\mathcal{Z}}
\newcommand{\TT}{\mathcal{T}}

\newcommand{\dfn}[1]{\textbf{\emph{#1}}}

\title{AeroMELD: A Linear Embedding of Aerosol Populations for Diagnostics and Latent Dynamics}

\author[1,4]{Ehsan~Saleh\thanks{Now at Amazon, Sunnyvale, California, USA; work on this paper was completed while at the University of Illinois Urbana-Champaign.}}
\author[1,4]{Saba~Ghaffari\,\orcidlink{0000-0003-0791-3905}}
\author[3]{Wenhan~Tang\,\orcidlink{0009-0009-5076-036X}}
\author[3]{Jeffrey~H.~Curtis\,\orcidlink{0000-0002-1447-2127}}
\author[5]{Lekha~Patel\,\orcidlink{0000-0003-3508-0672}}
\author[5]{Peter~A.~Bosler\,\orcidlink{0000-0002-3356-0296}}
\author[3]{Nicole~Riemer\,\orcidlink{0000-0002-3220-3457}}
\author[2]{Matthew~West\,\orcidlink{0000-0002-7605-0050}}

\affil[1]{Department of Computer Science, University of Illinois Urbana-Champaign, Urbana, Illinois, USA}
\affil[2]{Department of Mechanical Science and Engineering, University of Illinois Urbana-Champaign, Urbana, Illinois, USA}
\affil[3]{Department of Climate, Meteorology and Atmospheric Sciences, University of Illinois Urbana-Champaign, Urbana, Illinois, USA}
\affil[4]{National Center for Supercomputing Applications, University of Illinois Urbana-Champaign, Urbana, Illinois, USA}
\affil[5]{Center for Computing Research, Sandia National Laboratories, Albuquerque, New Mexico, USA}

\date{}

\begin{document}
\maketitle

\begin{abstract}
Accurately representing atmospheric aerosol populations is essential
for simulating aerosol–cloud interactions, radiative forcing, and ice
nucleation, yet existing reduced aerosol schemes impose structural
assumptions that limit their ability to capture composition diversity
and mixing state. Machine-learning approaches offer new opportunities
for flexible representations, but standard autoencoders do not
preserve the mathematical structure of aerosol populations and
therefore cannot support physically meaningful process operators. In
this paper, we introduce AeroMELD (Aerosol Measure Embedding for
Latent Dynamics), a mathematically grounded framework for constructing
low-dimensional latent variables that retain the intrinsic structure
of aerosol populations. We show that any permutation-invariant and
linear encoder must take a scale–shape decomposition, in which total
number concentration is represented explicitly and the latent shape is
a barycentric combination of per-particle embeddings. Taking this
aggregated representation as the latent state does not reduce
diagnostic expressiveness relative to a Deep Sets model; it moves the
nonlinear post-aggregation stage into the learned diagnostic map while
preserving latent linearity. Using particle-resolved data as ground truth, we encode weighted particle populations directly rather than binned aerosol states; size-resolved mass and number distributions are used only as diagnostic targets and visual summaries. This latent space enables accurate reconstruction of size-resolved mass and number distributions as well as CCN spectra, optical coefficients, and immersion-freezing behavior, while preserving the linear population structure needed for hybrid ML–physics aerosol models.
Although the experiments here focus on diagnostic reconstruction, the
embedding is designed as a latent state in which emissions and mixing
can be represented exactly and nonlinear microphysical processes can
be learned in a controlled latent space. This work establishes the
foundation for learning
aerosol-process evolution directly in latent space, which will be
explored in subsequent studies.
\end{abstract}

\section{Introduction}
Atmospheric aerosol populations influence climate through a wide range
of pathways, including direct radiative effects, cloud droplet
activation, and ice nucleation efficiency \citep{Poeschl2005, IPCC2021,
  Mcfiggans2006, Hoose2012}. These processes depend sensitively on
aerosol size, chemical composition, and mixing state
\citep{Mcfiggans2006, Riemer2019, Ching2017, Fierce2016}, all of which
vary across multiple dimensions and evolve under emissions, dilution,
coagulation, condensation and evaporation, chemical aging, and
deposition. Particle-resolved models can describe this complexity in
high detail by representing aerosol populations as size-composition
distributions over thousands of computational particles
\citep{Riemer2009,Gasparik2020}, but their computational cost prevents
their direct application in regional or global climate
models. Consequently, reduced aerosol representations are essential to
make physically realistic microphysics tractable at scale.

Existing reduced representations such as modal and sectional schemes
achieve computational efficiency by imposing strong structural
assumptions on the aerosol population (e.g., fixed lognormal shapes,
internally mixed modes) \citep{gelbard1980sectional, Whitby1997,
  Binkowski1995, Vignati2004, Bauer2008}. These assumptions constrain
the flexibility of the representation and can degrade performance when
microphysical variability, source diversity, or mixing-state effects
play a dominant role, such as in black-carbon aging, multicomponent
urban plumes, or INP-active mineralogy \citep{Ching2017, Fierce2016,
  yao2022quantifying, Tang2026}.

Recent work has begun to explore machine-learning approaches for
aerosol-state representation and aerosol-process emulation, including
graph-neural-network surrogates for particle-resolved aerosol dynamics
and generative latent representations of binned aerosol states
\citep{zheng2021estimating, wang2022learning,
  ferracina2025learning,saleh2025generative,saleh2025reconstructingaerosolstatepartial}. These
studies show that ML methods can provide compact and flexible aerosol
representations, but they do not enforce the linear population
structure needed for exact treatment of emissions, dilution, and
transport. In the first two papers of this series, we pursued latent embeddings of aerosol states using binned size–composition distributions rather than
particle-resolved data.
\citet{saleh2025generative} demonstrated that
variational autoencoders can learn compact latent spaces that preserve
CCN activity, optical properties, and immersion freezing spectra with
high fidelity.
\citet{saleh2025reconstructingaerosolstatepartial} developed
conditional generative models that infer full aerosol states from
partial observations, providing uncertainty-aware predictions across
multiple measurement configurations.
These studies showed that generative models can
efficiently represent binned aerosol populations and reconstruct them
from incomplete information, but they also highlighted the need for a
representation that preserves particle-level mixing state and respects
linear combination of aerosol populations.

While these generative models effectively represent binned aerosol states, they face two limitations that motivate the present work. First, binned size–composition distributions necessarily discard particle-level mixing-state information, even though mixing state strongly influences CCN activity, optical properties, and chemical aging. Second, the nonlinear autoencoders used in these studies do not preserve population addition: when two aerosol populations are mixed, the corresponding latent variables do not combine linearly. This is a fundamental limitation for atmospheric modeling because emissions, dilution, and transport require additive combination of populations. These shortcomings motivate a latent representation that preserves mixing state while respecting the linear structure of aerosol population addition.

\begin{figure}
	\centering
	\includegraphics[width=\textwidth]{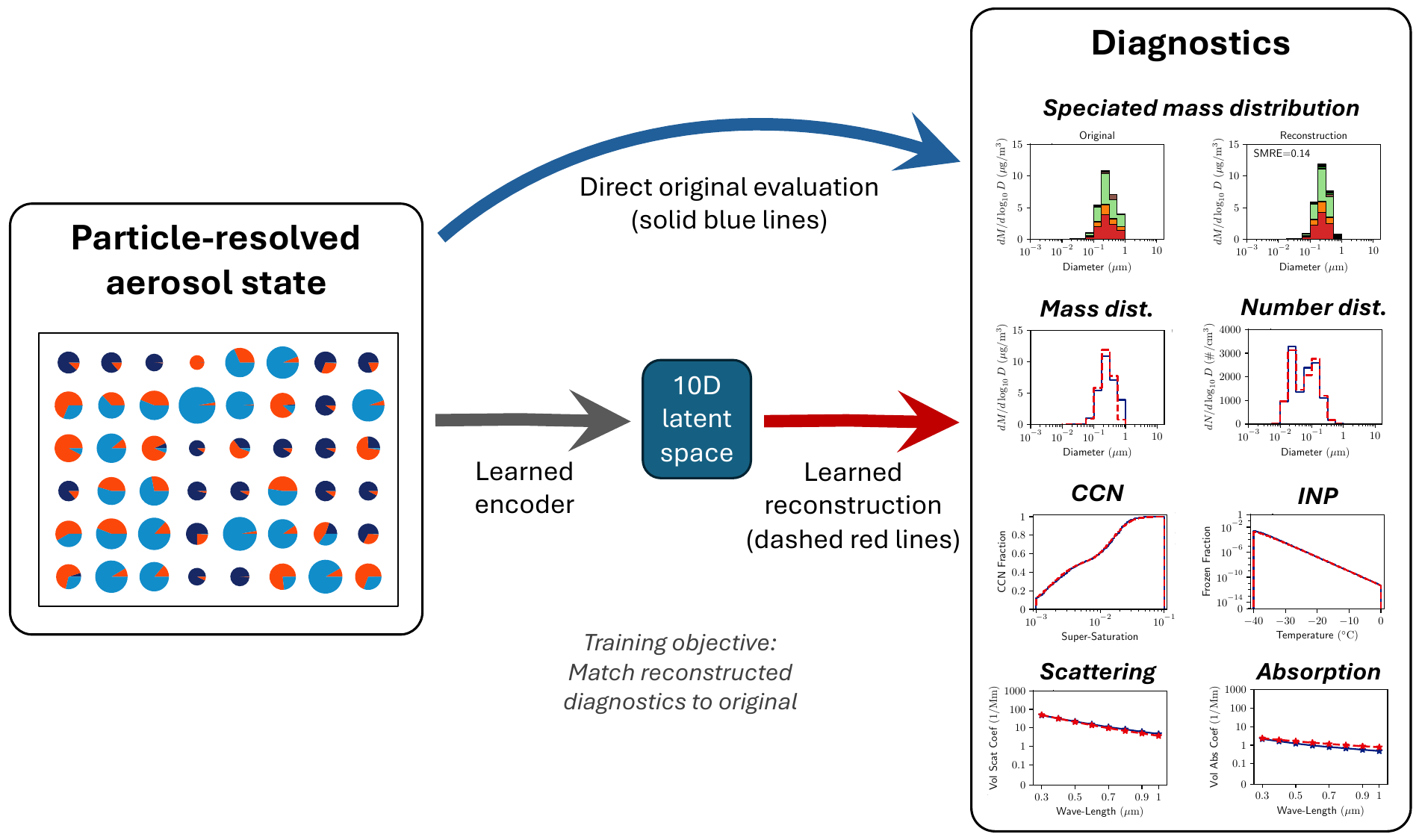}
	\caption{Overview of the AeroMELD framework.}
	\label{fig:01overview}
\end{figure}

In this third paper, we introduce \dfn{AeroMELD} (Aerosol Measure
Embedding for Latent Dynamics), a mathematically grounded
framework that encodes the particle-level aerosol
populations into a low-dimensional latent space (see Figure~\ref{fig:01overview}). We show that, after
choosing a scale coordinate proportional to total number concentration,
any encoder that is both permutation invariant and linear with
respect to the aerosol population must take a specific canonical form:
a \emph{scale-shape decomposition}, in which the total number
concentration is preserved explicitly and the latent ``shape'' is a
barycentric combination of per-particle embeddings. Equivalently, each
latent coordinate is a weighted sum of learned per-particle features;
in the measure-theoretic view developed below, these features define
learned test functions on composition space. Because this weighted
aggregation is independent of the number of particles used to
represent the population, the same trained encoder can be evaluated on
aerosol populations with any number of particles. This structure
yields several advantages. First, it ensures exact latent-space
representations of mixing and scalar concentration changes, allowing
emissions, dilution, and transport to be handled without learned
approximations. Second, it confines learned latent-process models to
the inherently nonlinear components of aerosol evolution (e.g.,
coagulation, gas-particle partitioning), greatly reducing model
complexity. Third, the separation of scale and shape improves
interpretability, prevents pathological latent drift, and enables
modular regularization strategies.

More broadly, by framing aerosol
populations as objects in a semimodule with physically meaningful
combination rules, AeroMELD contributes one of the first latent-space
formulations in the geosciences to enforce \emph{linear population
structure} at the representational level. This linear latent structure
does not require giving up the expressiveness of Deep Sets-style
diagnostic models~\citep{zaheer2017deep}: the nonlinear
post-aggregation stage can be treated as part of the learned latent
diagnostic rather than as part of the stored latent state. In the
experiments below, the base model applies this structure to 25000
aerosol population snapshots and maps a 16000-scalar sampled particle
representation of each population to a 10-dimensional latent state. We
note that the present
work does not attempt to learn the time evolution of aerosol
populations in latent space. Instead, we establish the mathematical
and representational foundations required for such latent-process
models, which will be developed in the next paper of this series.

Together, these developments provide a conceptual and practical
foundation for hybrid ML-physics aerosol simulators that combine exact
latent operators with learned nonlinear dynamics.  AeroMELD offers a
pathway toward reduced-order aerosol microphysics that is compact,
interpretable, and suitable for coupling with emerging ML-based
climate models and Earth system simulators.

Section~\ref{sec:theory} introduces the weighted-particle formulation
of aerosol populations and derives the linear, permutation-invariant
encoder structure underlying AeroMELD.
Section~\ref{sec:data_and_training} describes the particle-resolved
scenario library and the diagnostic operators used for training and
evaluation.  Section~\ref{sec:results} presents results on the learned
latent space, including the performance of latent diagnostics and
visualizations of the latent space structure.
Section~\ref{sec:conclusions} concludes with implications for future
reduced-order aerosol modeling and the integration of latent-space
microphysics into ML-physics hybrid climate frameworks.

\section{The AeroMELD Framework}
\label{sec:theory}

This section lays the mathematical foundations that define the AeroMELD framework. Crucially, we will see the precise sense in which the AeroMELD encoders are linear and how this can be used for aerosol modeling. The notation used in this paper is summarized in Table~\ref{tab:mathnotation}.

\begin{table}
\renewcommand{\arraystretch}{1}
	\centering
	\begin{tabular}{p{0.15\textwidth}p{0.7\textwidth}p{0.15\textwidth}}
	    \midrule
		Notation & Description & Space \\
		\midrule
		$A$ & The dimension of the particle composition space & $\mathbb{Z}$ \\
		$\alpha$ & A non-negative scalar & $\mathbb{R}_{\ge 0}$ \\
		$B$ & The number of diameter bins in the histograms & $\mathbb{Z}$ \\
		$\beta$ & A non-negative scalar & $\mathbb{R}_{\ge 0}$ \\
		$d$ & The dimension of the latent shape component ($d = L - 1$) & $\mathbb{Z}$ \\
		$D$ & The dimension of the diagnostic space & $\mathbb{Z}$ \\
		$\mathcal{D}$ & A diagnostic function & $\Pi \to \mathbb{R}^D$ \\
		$\hat{\mathcal{D}}$ & A scale-invariant or number-normalized diagnostic function & $\Pi \to \mathbb{R}^D$ \\
		$\tilde{\mathcal{D}}_\theta$ & A latent diagnostic function & $\mathbb{R}^L \to \mathbb{R}^D$ \\
		$\hat{\tilde{\mathcal{D}}}_\theta$ & The shape component of a latent diagnostic function & $\mathbb{R}^{L-1} \to \mathbb{R}^D$ \\
		$F$ & An aerosol process function & $\Pi \to \Pi$ \\
		$\tilde{F}_\theta$ & A latent aerosol process function & $\mathbb{R}^L \to \mathbb{R}^L$ \\
		$\mathcal{G}_\theta$ & A Deep Sets diagnostic head used in Section~\ref{sec:aeromeld-deepsets} & $\mathbb{R}^H \to \mathbb{R}^D$ \\
		$H$ & The auxiliary dimension after the Deep Sets post-aggregation map & $\mathbb{Z}$ \\
		$\mathcal{H}_\theta$ & A Deep Sets-style process model on the post-aggregation representation & $\mathbb{R}^H \to \mathbb{R}^H$ \\
		$L$ & The dimension of the latent space & $\mathbb{Z}$ \\
		$\lambda$ & A convex-combination weight & $[0,1]$ \\
		$\lambda$ & The optical wavelength & $\mathbb{R}$ \\
		$\mu$ & The variational latent mean & $\mathbb{R}^d$ \\
		$\mu_i$ & The composition of aerosol particle $i$ & $\mathbb{R}^A$ \\
		$n$ & The total number concentration of an aerosol population & $\mathbb{R}$ \\
		$n_i$ & The number concentration (weight) of particle $i$ & $\mathbb{R}$ \\
		$P$ & The dimension of the parameter space & $\mathbb{Z}$ \\
		$p_i$ & The normalized weight of particle $i$ & $\mathbb{R}$ \\
		$\Phi_\theta$ & The encoder & $\Pi \to \mathbb{R}^L$ \\
		$\hat{\Phi}_\theta$ & The shape encoder & $\hat{\Pi} \to \mathbb{R}^{L-1}$ \\
		$\phi_\theta$ & The per-particle encoder & $\mathbb{R}^A \to \mathbb{R}^{L-1}$ \\
		$\Pi$ & The space of aerosol populations $\pi$ & --- \\
		$\hat{\Pi}$ & The space of normalized aerosol populations $\hat{\pi}$ & --- \\
		$\pi$ & A weighted aerosol population $\{(n_i,\mu_i)\}_{i=1}^N$ & $\Pi$ \\
		$\hat{\pi}$ & A normalized aerosol population $\{(p_i,\mu_i)\}_{i=1}^N$ & $\hat{\Pi}$ \\
		$\rho_\theta$ & The post-aggregation Deep Sets map used in Section~\ref{sec:aeromeld-deepsets} & $\mathbb{R}^L \to \mathbb{R}^H$ \\
		$s$ & The scaled latent shape in the double-scale representation & $\mathbb{R}^{L-1}$ \\
		$s$ & The critical supersaturation level & $\mathbb{R}$ \\
		$\Sigma$ & The variational latent covariance matrix & $\mathbb{R}^{d \times d}$ \\
		$\sigma_i$ & The variational latent standard deviation for dimension $i$ & $\mathbb{R}$ \\
		$T$ & The temperature & $\mathbb{R}$ \\
		$\TT$ & the pre-processing transformation & --- \\
		$\theta$ & The learned parameters & $\mathbb{R}^{P}$ \\
		$w_\mu$ & The KL divergence weight for the mean term & $\mathbb{R}$ \\
		$w_\sigma$ & The KL divergence weight for the variance term & $\mathbb{R}$ \\
		$\Z$ & The standardization operator applying a zero-mean and unit-scaling transformation inferred over the training data & --- \\
		$z$ & The shape component of the latent representation & $\mathbb{R}^{L-1}$ \\
	\end{tabular}
	\caption{The mathematical notation used throughout the paper.}
	\label{tab:mathnotation}
\end{table}

\subsection{Composition Space and Aerosol Populations}

An \dfn{aerosol particle} is represented by its \dfn{composition} $\mu \in \mathbb{R}^A$, where $A$ is the number of chemical species considered. Each component $\mu^a$ represents the mass of species $a$ in the particle, such as sulfate, nitrate, organic carbon, etc. The space $\mathbb{R}^A$ is referred to as the \dfn{composition space}. It is straightforward to extend this representation to include other particle properties, such as particle charge or fractal dimension, by adding additional components to the composition.

An \dfn{aerosol population} is a measure $\pi$ on the composition space, representing the number distribution of particles with different compositions. We denote the space of aerosol populations by $\Pi$. In practice, we represent an aerosol population as a weighted set of particles $\pi = \{(n_i, \mu_i)\}_{i=1}^N$, where $n_i$ is the number concentration of particles with composition $\mu_i$. Technically, we use multisets in the sense of \citet{knuth1997seminumerical}, so that particles with the same composition can appear multiple times in the set.

A \dfn{normalized aerosol population} has weights which sum to one. That is, it has the form $\hat{\pi} = \{(p_i, \mu_i)\}_{i=1}^N$, where $\sum_i p_i = 1$. This can be regarded as a probability measure on the composition space. Given an aerosol population, the \dfn{total number concentration} is $n = \sum_i n_i$ and the corresponding normalized population has weights $p_i = \frac{n_i}{n}$. Equivalently, the total number concentration is the integral of the population measure over the composition space and the normalized population is the probability measure obtained by dividing the population measure by its integral.

We can think of an aerosol population $\pi$ as being composed of two parts: a \dfn{scale component} given by its total number concentration $n$ and a \dfn{shape component} given by its normalized population $\hat{\pi}$. We can thus represent an aerosol population as the pair $(n, \hat{\pi})$. This \dfn{scale-shape representation} will be useful for constructing encoders and diagnostics that respect the linear structure of aerosol populations.

\subsection{Linear Structures on Aerosol Populations}

An important feature of aerosol populations is that they have a \dfn{linear structure}. Physically, the addition or union of two aerosol populations corresponds to mixing them together, and scaling an aerosol population corresponds to changing its overall concentration without changing its composition distribution. These properties are crucial in 3D numerical models for implementing emissions, which add an emitted population into an existing aerosol, and transport, which takes linear combinations of aerosols from neighboring grid cells.

Mathematically, given two aerosol populations $\pi_1$ and $\pi_2$, we can form their \dfn{union} $\pi_1 + \pi_2$ by combining the weighted sets of particles (technically, this is a multiset union $\uplus$ as defined by \citet{knuth1997seminumerical}). We can also \dfn{scale} an aerosol population $\pi$ by a positive scalar $\alpha$ to obtain a new population $\alpha \pi$ with weights $\alpha n_i$. If we are thinking of aerosol populations as measures, then the union corresponds to measure addition and scaling corresponds to multiplying the measure by a scalar. We think of this as a linear structure on the space of aerosol populations $\Pi$, although technically it is an $\mathbb{R}_{\ge 0}$-semimodule rather than a vector space, since we only allow scaling by non-negative scalars and don't have additive inverses.

In the scale-shape representation, scalar multiplication and addition of aerosol populations can be expressed as
\begin{align}
	\label{eqn:population-scalar-multiplication}
	\alpha (n, \hat{\pi}) &= (\alpha n, \hat{\pi}), \\
	\label{eqn:population-addition}
	(n_1, \hat{\pi}_1) + (n_2, \hat{\pi}_2) &= \bigl(n_1 + n_2, \lambda \hat{\pi}_1 + (1 - \lambda) \hat{\pi}_2\bigr), \text{ where } \lambda = \frac{n_1}{n_1 + n_2}.
\end{align}
This is a \dfn{barycentric} or \dfn{convex} addition on the shape components, which reflects the fact that the normalized population of a mixture is a weighted average of the normalized populations of its components. Together these give linear combinations as
\begin{equation}
\label{eqn:population-linear}
\alpha (n_1, \hat{\pi}_1) + \beta (n_2, \hat{\pi}_2)
= \bigl(\alpha n_1 + \beta n_2, \lambda \hat{\pi}_1 + (1 - \lambda) \hat{\pi}_2\bigr), \text{ where } \lambda = \frac{\alpha n_1}{\alpha n_1 + \beta n_2}.
\end{equation}

\subsection{Linear Set Encoders}

We want to construct an \dfn{encoder} $\Phi_\theta : \Pi \to \mathbb{R}^L$ that maps aerosol populations to a \dfn{latent space} $\mathbb{R}^L$, parameterized by $\theta$. The latent space is a low-dimensional representation of the aerosol population, where $L$ is much smaller than the dimension of the original composition space or the number of particles in the population. We want our encoder to have three key properties:
\begin{enumerate}
    \item \dfn{Permutation invariance:} $\Phi_\theta(\pi) = \Phi_\theta(\pi')$ for any permutation $\pi'$ of the particles in $\pi$. This ensures the encoder really acts on the aerosol population as a whole, rather than on the specific ordering of particles in the representation.
    \item \dfn{Linearity:} $\Phi_\theta(\alpha \pi_1 + \beta \pi_2) = \alpha\Phi_\theta(\pi_1) + \beta\Phi_\theta(\pi_2)$, for any positive scalars $\alpha$ and $\beta$ and aerosol populations $\pi_1$ and $\pi_2$. This is important for efficient representation of aerosol processes that involve mixing and scaling of populations, such as emissions and transport.
    \item \dfn{Scale-shape representation:} The latent space should separate scale and shape components as $(n,z) \in \mathbb{R}^L$, so that the first component $n \in \mathbb{R}$ is the total number concentration (the \dfn{scale component}) and the remaining $z \in \mathbb{R}^{L-1}$ is the \dfn{latent shape component}. This allows us to regularize and manipulate scale and shape independently in the latent space. We denote the shape component of the encoder by $\hat{\Phi}_\theta$, so that, for $\pi=(n,\hat{\pi})$, we write $\Phi_\theta(\pi)=(n,\hat{\Phi}_\theta(\hat{\pi}))$. The theorem below characterizes the form imposed on $\hat{\Phi}_\theta$ by linearity and permutation invariance.
\end{enumerate}
The scale-shape representation in the latent space means that scalar multiplication and addition are defined as
\begin{align}
	\label{eqn:latent-scalar-multiplication}
	\alpha (n, z) &= (\alpha n, z), \\
	\label{eqn:latent-addition}
	(n_1, z_1) + (n_2, z_2) &= \bigl(n_1 + n_2, \lambda z_1 + (1 - \lambda) z_2\bigr), \text{ where } \lambda = \frac{n_1}{n_1 + n_2}.
\end{align}
This also technically makes the latent space a $\mathbb{R}_{\ge 0}$-semimodule rather than a vector space and linearity of the encoder should technically be interpreted as semilinearity (a semimodule homomorphism). Combining the above expressions gives the general linear combination
\begin{equation}
\label{eqn:latent-linear}
\alpha (n_1, z_1) + \beta (n_2, z_2)
= \bigl(\alpha n_1 + \beta n_2, \lambda z_1 + (1 - \lambda) z_2\bigr), \text{ where } \lambda = \frac{\alpha n_1}{\alpha n_1 + \beta n_2}.
\end{equation}

The three properties we require of our encoder turn out to restrict it
to a simple deterministic form: a learned nonlinear feature map is
applied to each particle and the resulting features are averaged using
the normalized particle weights.

\begin{theorem}
	\label{thm:encoder-form}
  Every deterministic encoder $\Phi_\theta : \Pi \to \mathbb{R}^L$ that is permutation invariant and linear can be expressed (up to rescalings in the latent scale component) in the form
\begin{align}
	\label{eqn:ecoder-form-scale}
	(n,z) &= \Phi_\theta(\pi) = \bigl(n, \hat{\Phi}_\theta(\hat{\pi})\bigr), & & \text{where } \pi = (n, \hat{\pi}), \\
	\label{eqn:encoder-form-shape}
	z &= \hat{\Phi}_\theta(\hat{\pi}) = \sum_{i=1}^N p_i \phi_\theta(\mu_i), & & \text{where } \hat{\pi} = \{(p_i, \mu_i)\}_{i=1}^N.
\end{align}
Here $\phi_\theta: \mathbb{R}^A \to \mathbb{R}^{L-1}$ is a learned function that maps the composition vector to the $(L-1)$-dimensional latent shape component.
\end{theorem}

\begin{proof}
	It is clear that any encoder of the form~(\ref{eqn:ecoder-form-scale})--(\ref{eqn:encoder-form-shape}) is permutation invariant and it is straightforward to check that it is linear with respect to the latent space operations (see \ref{sec:direct-proof-linearity}). Technically, it is a semimodule homomorphism with respect to the semimodule structures on the population space $\Pi$, given by (\ref{eqn:population-scalar-multiplication})--(\ref{eqn:population-addition}), and the latent space $\mathbb{R}^L$, given by (\ref{eqn:latent-scalar-multiplication})--(\ref{eqn:latent-addition}).

	The necessity of the form~(\ref{eqn:ecoder-form-scale}) follows from the fact that scalar multiplication only affects the first component, and by linearity we can choose the first component to be the total number concentration $n$. The form~(\ref{eqn:encoder-form-shape}) follows because, for finite weighted particle populations, linearity makes the scaled shape coordinate additive over particles, while permutation invariance ensures that this additive contribution depends only on each particle composition and weight, not on particle ordering.
\end{proof}

Theorem~\ref{thm:encoder-form} characterizes the latent aerosol state
used for physical operations such as mixing, emissions, and
transport. The aggregation in Eq.~(\ref{eqn:encoder-form-shape}) can
also be read as the finite-particle evaluation of
$z=\int \phi_\theta(\mu)\,\diff\hat{\pi}(\mu)$, and therefore is not
tied to a fixed number of particles. If a sequence of empirical
particle measures converges weakly to a limiting normalized aerosol
measure, and the learned component functions of $\phi_\theta$ are
bounded and continuous on the composition domain, then
the AeroMELD shape coordinates converge to the corresponding integrals
against the limiting measure. The VAE regularization introduced below adds a stochastic
training distribution around this deterministic shape coordinate, but
the covariance of that distribution is an auxiliary training quantity
and is not part of the AeroMELD state itself.

The form of the encoder in Theorem~\ref{thm:encoder-form} is closely related to the Deep Sets architecture of \citet{zaheer2017deep}, which is designed to be permutation invariant. The precise relationship is discussed in Section~\ref{sec:aeromeld-deepsets}.

\subsection{Double-Scale Representation}

The scale-shape representation separates the total number concentration from the normalized population, but it can be useful for the latent operations to re-introduce the scale into the shape component. This leads to a \dfn{double-scale representation} where the scale is present in both components. We define the double-scale latent variable as $(n, s) \in \mathbb{R}^L$, where $s = n z$ is the \dfn{scaled latent shape component}. The latent operations in this representation are much simpler, as scalar multiplication and addition become
\begin{equation}
\label{eqn:double-scale-linear}
\alpha (n_1, s_1) + \beta (n_2, s_2)
= (\alpha n_1 + \beta n_2, \alpha s_1 + \beta s_2).
\end{equation}
See \ref{sec:double-scale-derivation} for a proof of this relationship between the scale-shape and double-scale representations.

While the operations for the double-scale representation are simpler, the scale-shape representation has advantages for interpretability and regularization, since it separates scale and shape explicitly. In this paper, we will primarily use the scale-shape representation for the latent variables, but the double-scale representation will be useful for transport in Section~\ref{sec:transport}.

\subsection{Diagnostic Functions}

The linear structure also carries over to aerosol \dfn{diagnostics}, which are typically linear functionals of the aerosol population. For example, the speciated mass distribution with $B$ bins is a map $\mathcal{D}_{\rm SMD} : \Pi \to \mathbb{R}^{A \times B}$ defined by summing the contributions of each particle to the mass in each bin. Similarly, the number distribution is $\mathcal{D}_{\rm ND} : \Pi \to \mathbb{R}^B$. Each of these diagnostics $\mathcal{D}$ is linear in the aerosol population, so that for any two populations $\pi_1$ and $\pi_2$ and positive scalars $\alpha$ and $\beta$, we have
\begin{equation}
	\mathcal{D}(\alpha \pi_1 + \beta \pi_2) = \alpha \mathcal{D}(\pi_1) + \beta \mathcal{D}(\pi_2).
\end{equation}

Some diagnostics are \dfn{scale-invariant}, or \dfn{number-normalized}, diagnostics. These depend on the normalized population shape rather than on the total number concentration, so that $\hat{\mathcal{D}}(\alpha \pi) = \hat{\mathcal{D}}(\pi)$ for any positive scalar $\alpha$. Many such diagnostics can be obtained by dividing a linear diagnostic by the total number concentration, for example by computing a CCN fraction as the activated number concentration divided by the total number concentration. In this case, such a diagnostic $\hat{\mathcal{D}}$ satisfies the convex combination
\begin{equation}
	\label{eqn:normalized-diagnostic}
	\hat{\mathcal{D}}(\alpha \pi_1 + \beta \pi_2) = \lambda \hat{\mathcal{D}}(\pi_1) + (1 - \lambda) \hat{\mathcal{D}}(\pi_2), \text{ where } \lambda = \frac{\alpha n_1}{\alpha n_1 + \beta n_2},
\end{equation}
and $n_1$ and $n_2$ are the total number concentrations of $\pi_1$ and $\pi_2$, respectively.

\subsection{Latent Diagnostics}

We want to learn \dfn{latent diagnostic} maps $\tilde{\mathcal{D}}_\theta : \mathbb{R}^L \to \mathbb{R}^D$ that map the latent space to diagnostic variables $y \in \mathbb{R}^D$. For example, we will learn latent diagnostics for the speciated mass distribution, number distribution, CCN spectrum, optical properties, or ice nucleation properties. While we could require the latent diagnostics to be linear like the true diagnostics, we will use a more expressive structure by only restricting them to be homogeneous in $n$, meaning they can be factored as
\begin{equation}
	\tilde{\mathcal{D}}_\theta(n, z) = n \hat{\tilde{\mathcal{D}}}_\theta(z),
\end{equation}
where $\hat{\tilde{\mathcal{D}}}_\theta : \mathbb{R}^{L-1} \to \mathbb{R}^D$ is a learned function that maps the latent shape component to the diagnostic variables. In the case of a scale-invariant diagnostic that satisfies~(\ref{eqn:normalized-diagnostic}), we will have
\begin{equation}
	\tilde{\mathcal{D}}_\theta(n, z) = \hat{\tilde{\mathcal{D}}}_\theta(z).
\end{equation}
If we did wish to constrain $\tilde{D}_\theta$ to be linear, $\hat{\tilde{D}}_\theta$ would be a linear map $\mathbb{R}^{L-1} \to \mathbb{R}^D$.

\subsection{Reconstruction Loss Function}

We want to find an encoder and latent diagnostic models that approximate the true diagnostics. This is expressed as the diagram
\begin{equation}
\begin{tikzcd}
  & & y_t \\
  \pi_t \arrow[r, mapsto, "\Phi_\theta"] \arrow[rru, mapsto, bend left=20, "\mathcal{D}"]
  & (n_t, z_t) \arrow[r, mapsto, "\tilde{\mathcal{D}}_\theta"]
  & \tilde{y}_t \arrow[u, dashed, leftrightarrow, "\approx"']
\end{tikzcd}
\end{equation}
Based on this, we have the \dfn{reconstruction loss function}
\begin{align}
	\label{eq:loss_reconstruction}
\mathcal{L}_\text{reconstruct}(\theta) &= \mathbb{E}_\pi \| \mathcal{D}(\pi) - \tilde{\mathcal{D}}_\theta(\Phi_\theta(\pi)) \|^2 \\
&= \mathbb{E}_\pi \| y_i - \tilde{y}_i \|^2,
\end{align}
where $y_i$ are the true diagnostics of population $i$ and $\tilde{y}_i$ are the approximated versions.

\subsection{Population Dynamics}
\label{sec:dynamics}

The empirical focus of this paper is diagnostic reconstruction, but the purpose of the AeroMELD embedding is to provide a latent state on which aerosol dynamics can act. We briefly sketch how dynamics are incorporated to provide a more complete picture, deferring detailed exploration to future work. Figure~\ref{fig:aeromeld-commutative} shows how the AeroMELD framework includes aerosol population dynamics. These dynamics could include coagulation, gas/particle partitioning, dry deposition, or other processes that directly affect the aerosol population. We assume that the dynamics is given by a map $F : \Pi \to \Pi$, which would generally also depend on environmental variables or other system states such as gas concentrations.

\begin{figure}
	\centering
\begin{tikzcd}
  \text{time } t
  & \Pi \arrow[r, "\Phi_\theta"] \arrow[rr, bend left=50, "\mathcal{D}"] \arrow[d, "F"]
  & \mathbb{R}^L \arrow[r, "\tilde{\mathcal{D}}_\theta"] \arrow[d, "\tilde{F}_\theta"]
  & \mathbb{R}^D \\
  %%%%%%%%%%%%%%%%%%%%%%%%%%%%%%%%%%%%%%%%%%%%%%%%%%%%%%%%%%%%%%%%%%%%%%%%%%%%%%%%%%
  \text{time } t+1
  & \Pi \arrow[r, "\Phi_\theta"] \arrow[d, "F"]
  & \mathbb{R}^L \arrow[r, "\tilde{\mathcal{D}}_\theta"] \arrow[d, "\tilde{F}_\theta"]
  & \mathbb{R}^D \\
  %%%%%%%%%%%%%%%%%%%%%%%%%%%%%%%%%%%%%%%%%%%%%%%%%%%%%%%%%%%%%%%%%%%%%%%%%%%%%%%%%%
  & \vdots
  & \vdots
  & \vdots
\end{tikzcd}
\caption{AeroMELD framework with learned encoder $\Phi_\theta$ from the population space $\Pi$ to the latent space $\mathbb{R}^L$, true and learned diagnostics $\mathcal{D}$ and $\tilde{\mathcal{D}}_\theta$ mapping to the diagnostic space $\mathbb{R}^D$, and true and learned time evolutions $F$ and $\tilde{F}_\theta$. This is not a commutative diagram but rather indicates approximate equivalence. The learned maps are trying to make this diagram as close to commutative as possible. \label{fig:aeromeld-commutative}}
\end{figure}
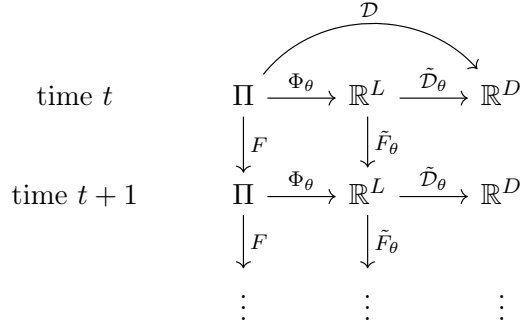

The dynamics can be approximated by a learned map $\tilde{F}_\theta : \mathbb{R}^L \to \mathbb{R}^L$ (generally one time step of a neural ODE~\citep{chen2018neural}) on the latent space, which motivates the dynamics loss
\begin{equation}
\mathcal{L}_\text{dynamics}(\theta) = \mathbb{E}_\pi \| \Phi_\theta(F(\pi)) - \tilde{F}_\theta(\Phi_\theta(\pi)) \|^2.
\end{equation}

While nonlinear dynamics $F$ must be approximated by a learned map $\tilde{F}_\theta$, the situation for linear dynamics is much simpler. Linear dynamics, such as aerosol transport and emissions, can be evaluated directly with no error using the linearity of the AeroMELD latent representation. Table~\ref{tab:processes} summarizes the process representations.

\begin{table}
\begin{tabular}{ll}
Aerosol Process & Representation in AeroMELD \\
\hline
Initial conditions & Direct evaluation of encoder $\Phi_\theta$ (\S~\ref{sec:initial_conditions_and_emissions}) \\
Emissions & Combine latent representations using linear structure (\S~\ref{sec:initial_conditions_and_emissions}) \\
Transport & Combine latent representations using linear structure (\S~\ref{sec:transport}) \\
Coagulation & Approximate learned function $\tilde{F}_{\text{coag},\theta}$ (\S~\ref{sec:dynamics}) \\
Deposition & Approximate learned function $\tilde{F}_{\text{dep},\theta}$ (\S~\ref{sec:dynamics}) \\
Gas/aerosol partitioning & Approximate learned function $\tilde{F}_{\text{gas},\theta}$ (\S~\ref{sec:dynamics})
\end{tabular}
\caption{Overview of aerosol process descriptions in the AeroMELD framework. Note that these processes are described in this paper but their implementation will be explored in future work. \label{tab:processes}}
\end{table}

\subsection{Initial Conditions and Emissions}
\label{sec:initial_conditions_and_emissions}

In an atmospheric model, initial conditions and emissions are typically specified as a linear combination of aerosol populations, each of which has a given binned, modal, or particle-resolved representation. In the AeroMELD framework, we can convert each of these constituent aerosol populations into a latent AeroMELD representation and then take linear combinations as needed in the latent space. This can be done by first sampling particles for each population (if not already in particle-resolved form), then directly evaluating the encoder $\Phi_\theta$ to obtain their latent representations.

For example, suppose we have emissions from $M$ different sectors, each being emitted at a time-dependent rate. Assume $\hat{\pi}_1$, \ldots, $\hat{\pi}_M$ are the normalized aerosol population shapes for each sector, given in particle-resolved form (sampled from binned or modal representations if needed), and $n_1(t), \ldots, n_M(t)$ are their total emitted number concentrations at time $t$. Before the simulation starts, we can do a one-off evaluation of the shape encoder to obtain latent shape representations $z_1$, \ldots, $z_M$, where $z_m = \hat{\Phi}_\theta(\hat{\pi}_m)$ for $m = 1, \ldots, M$. At time $t$, the emitted aerosol population has latent representation
\begin{align}
	n_{\rm emit}(t) &= \sum_{m=1}^M n_m(t), \\
	z_{\rm emit}(t) &= \frac{1}{n_{\rm emit}(t)} \sum_{m=1}^M n_m(t) z_m.
\end{align}
Because the AeroMELD encoder is linear, this combination of sector emissions is exact, and no particle data needs to be stored or re-encoded at each time step. These latent emissions can then be added to the model state using the latent addition operation.

\subsection{Transport (Advection and Diffusion)}
\label{sec:transport}

Transport schemes in atmospheric models typically involve linear operations on aerosol populations, such as advection and diffusion~\citep{durran2010numerical}. In the AeroMELD framework, these linear transport operations can be directly applied in the latent space using the linear structure of the latent representation. For example, if grid cell $k$ at time $t$ has aerosol population represented by latent variable $(n_k^t, z_k^t)$, and the transport scheme computes a linear combination of populations from neighboring grid cells, we can perform the same linear combination in the latent space. This allows for efficient and exact representation of transport processes without needing to reconstruct particle data.

It is especially simple to implement transport processes in the double-scale representation, where both components contain scale information. In this representation, transport operations correspond to standard linear combinations of the latent variables, as shown in Equation~(\ref{eqn:double-scale-linear}). Because of this, the double-scale latent variables can be transported as regular tracers in the atmospheric model, simplifying the implementation. More precisely, before computing transport, we convert the scale-shape latent variables $(n_k^t, z_k^t)$ to double-scale variables $(n_k^t, s_k^t)$, where $s_k^t = n_k^t z_k^t$. We then allow the transport scheme to operate on $(n_k^t, s_k^t)$ as regular tracers. After transport, we convert back to scale-shape representation by computing $z_k^{t+1} = s_k^{t+1} / n_k^{t+1}$.

Note that the use of limiters or other nonlinear stabilization techniques in transport schemes may introduce small nonlinearities. This will mean that using double-scale representation may not be exactly equivalent to transporting the latent variables in scale-shape representation. However, these nonlinearities are typically small.

\subsection{Measure-Theoretic Interpretation}

The AeroMELD framework can be understood entirely in terms of finite
weighted particle populations. For readers who prefer a
measure-theoretic formulation, the same construction can be viewed as
a learned measure embedding. The surrounding linear object is the
vector space of finite signed measures on composition space; aerosol
populations occupy the nonnegative cone inside this vector space. The
learned per-particle map
$\phi_\theta : \mathbb{R}^A \to \mathbb{R}^{L-1}$ defines a collection
of learned test functions
\begin{equation}
	b_{\theta,\ell}(\mu) = [\phi_\theta(\mu)]_\ell,
	\qquad \ell = 1,\ldots,L-1,
\end{equation}
and the AeroMELD coordinates are obtained by pairing these functions
with the aerosol population measure,
\begin{equation}
	z_\ell
	= \int b_{\theta,\ell}(\mu)\,\diff\hat{\pi}(\mu),
	\qquad
	s_\ell
	= n z_\ell
	= \int b_{\theta,\ell}(\mu)\,\diff\pi(\mu).
\end{equation}
Thus the latent shape coordinates are learned linear functionals of
the normalized aerosol measure, while the double-scale coordinates are
the corresponding linear functionals of the unnormalized population
measure.

For a finite weighted particle representation, these measure pairings
reduce to weighted sums,
\begin{equation}
	z_\ell = \sum_{i=1}^N p_i b_{\theta,\ell}(\mu_i),
	\qquad
	s_\ell = \sum_{i=1}^N n_i b_{\theta,\ell}(\mu_i).
\end{equation}
This gives a simple finite-vector intuition. If measures were
restricted to a fixed support $\{\xi_j\}_{j=1}^K$, then a population
$\pi=\sum_{j=1}^K w_j\delta_{\xi_j}$ could be identified with the
weight vector $w\in\mathbb{R}^K$. A learned test function
$b_{\theta,\ell}$ would then be represented by the vector
$(b_{\theta,\ell}(\xi_1),\ldots,b_{\theta,\ell}(\xi_K))$, and the
coordinate $s_\ell$ would be the usual dot product with $w$. AeroMELD
generalizes this construction by learning functions on composition
space itself, so the support particles do not need to lie on a fixed
grid or have fixed cardinality.

In this sense, AeroMELD is similar in spirit to traditional basis
representations such as modal~\citep{Whitby1997} or
sectional~\citep{gelbard1980sectional} methods, which remain
foundational in aerosol modeling (see also~\citep{Riemer2019} for a
modern review). The key distinction is that AeroMELD learns the
functions that define the population coordinates directly from data,
rather than imposing fixed modal shapes, fixed size bins, or a fixed
particle support. This allows the learned coordinates to represent
complex mixing states and composition distributions while remaining
well-defined for any finite particle representation of the population.

This perspective also clarifies the connection between AeroMELD and
other dimensionality reduction techniques, such as principal component
analysis (PCA)~\citep{jolliffe2002pca}, which is closely related to
proper orthogonal decomposition (POD)~\citep{holmes2012turbulence} and
empirical orthogonal functions (EOF)~\citep{hannachi2007eofreview}.
PCA, POD, and EOF methods typically begin by representing each sample
as a vector in a common finite coordinate system, such as a fixed
grid, binning, or support. Their basis vectors are therefore tied to
that chosen discretization. AeroMELD instead learns functions on
composition space and evaluates those functions against each aerosol
measure. The resulting coordinates can be computed for particle
populations with different supports or different numbers of particles.
The scale-shape decomposition and the diagnostic-focused training
objective are additional advantages of this construction.

\subsection{Expressiveness Relative to Deep Sets}
\label{sec:aeromeld-deepsets}

Deep Sets provide a canonical universal architecture for
permutation-invariant functions on finite sets: under appropriate
assumptions, such functions can be represented or approximated by
applying a nonlinear map $\rho_\theta$ to a sum of learned per-element
features $\phi_\theta$~\citep{zaheer2017deep}. It is thus natural to ask
whether the linearity constraint on the AeroMELD latent state sacrifices
expressiveness relative to Deep Sets. The answer is no: for diagnostic reconstruction,
AeroMELD simply changes where the boundary is placed between the
encoder and the downstream learned maps.

In the notation of this paper, the AeroMELD encoder $\Phi_\theta : \Pi
\to \mathbb{R}^L$ contains the permutation-invariant weighted
aggregation over particles. A conventional Deep Sets model would
typically add a nonlinear post-aggregation map $\rho_\theta :
\mathbb{R}^L \to \mathbb{R}^H$ and then a diagnostic head
$\mathcal{G}_\theta : \mathbb{R}^H \to \mathbb{R}^D$. We can
factorize this in two equivalent ways:
\begin{equation}
\label{eqn:aeromeld-deepsets-factorization}
\begin{aligned}
&
\overbrace{
\Pi \xrightarrow{\Phi_\theta} \mathbb{R}^{L}
}^{\substack{\text{AeroMELD}\\\text{encoder}}}
\overbrace{
\xrightarrow{\rho_\theta} \mathbb{R}^{H}
\xrightarrow{\mathcal{G}_\theta} \mathbb{R}^{D}
}^{\substack{\text{AeroMELD learned}\\\text{diagnostic } \tilde{\mathcal{D}}_\theta}}
\\[0.75em]
&
\underbrace{
\Pi \xrightarrow{\Phi_\theta} \mathbb{R}^{L}
\xrightarrow{\rho_\theta} \mathbb{R}^{H}
}_{\substack{\text{Deep Sets}\\\text{encoder}}}
\underbrace{
\xrightarrow{\mathcal{G}_\theta} \mathbb{R}^{D}
}_{\substack{\text{Deep Sets}\\\text{learned}\\\text{diagnostic}}} .
\end{aligned}
\end{equation}
A Deep Sets interpretation would call $\rho_\theta \circ \Phi_\theta$ the encoder, $\mathbb{R}^H$ the latent space, and $\mathcal{G}_\theta$ the learned diagnostic map. In contrast, AeroMELD calls $\Phi_\theta$ the encoder, $\mathbb{R}^L$ the latent space, and absorbs the post-aggregation map into the latent diagnostic,
\begin{equation}
	\tilde{\mathcal{D}}_\theta
	= \mathcal{G}_\theta \circ \rho_\theta.
\end{equation}
Thus the same composed map from aerosol populations to diagnostics is represented in either case. The difference is that AeroMELD stores the linear aggregated representation $\Phi_\theta(\pi)$ as the latent state, rather than the nonlinear transformed representation $\rho_\theta(\Phi_\theta(\pi))$.

This distinction is important for aerosol modeling. If the nonlinear post-aggregation map $\rho_\theta$ were included inside the stored latent state, then the resulting latent variable would generally no longer preserve the linear structure of aerosol population addition. By keeping the latent state at $\mathbb{R}^L$, AeroMELD retains exact linear operations for emissions and transport, while still allowing arbitrary nonlinear maps to act downstream when predicting diagnostics.

A similar perspective applies to learned latent evolution operators, although exact equivalence requires an additional qualification. If a Deep Sets-style process model evolves the post-aggregation representation through a map $\mathcal{H}_\theta : \mathbb{R}^H \to \mathbb{R}^H$, then the corresponding AeroMELD latent process would formally have the structure
\begin{equation}
	\tilde{F}_\theta
	= \rho_\theta^\dagger \circ \mathcal{H}_\theta \circ \rho_\theta,
\end{equation}
where $\rho_\theta^\dagger : \mathbb{R}^H \to \mathbb{R}^L$ is an inverse, right-inverse, or learned return map on the relevant latent manifold. In practice, this means that the nonlinear post-aggregation map $\rho_\theta$ could be used as a shared backbone for both latent diagnostics and learned process models, with separate output heads for diagnostics and evolution. The stored model state would nevertheless remain the AeroMELD latent variable in $\mathbb{R}^L$, preserving the linear operations that motivate the framework.

\section{Data and Training}
\label{sec:data_and_training}

This section details the dataset used for training and evaluating the AeroMELD models for diagnostics. We begin by describing the source of the data, which is a comprehensive library of aerosol scenarios generated by a particle-resolved model. Following this, we outline the calculation of key climate-relevant aerosol diagnostic variables, including CCN spectra, optical properties, and ice nucleation activity. We also describe the methodology for splitting the data into training and testing sets to ensure robust model evaluation, and describe the details of the training procedure.

\subsection{Data Source}

The dataset used in this study is sourced from the scenario library detailed in~\citet{Gasparik2020}. This library was generated using the particle-resolved aerosol box model PartMC-MOSAIC~\citep{Riemer2009,Zaveri2008}. PartMC explicitly tracks the composition and size of thousands of individual computational particles within an evolving population, resolving mixing state and allowing for a detailed representation of aerosol microphysics. Particle coagulation is simulated using a stochastic Monte Carlo approach, while MOSAIC provides the coupled gas- and aerosol-phase chemistry and thermodynamics. Together, this framework captures emissions, coagulation, dilution with the background, and gas–aerosol partitioning, producing a comprehensive dataset of aerosol populations across diverse atmospheric conditions and emission scenarios.

The library comprises 1000 distinct scenarios, each corresponding to a 24-hour simulation. For the dataset used here, we retain the initial state together with 24 hourly outputs, yielding 25 time snapshots per scenario and 25000 aerosol population snapshots in total. The aerosol populations within these scenarios are described by 15 chemical species yielding particle-resolved ensembles. For the neural-network inputs used here, each population snapshot is represented by 1000 weighted particles.
Each particle carries a 15-dimensional chemical composition vector and one scalar weight, so the sampled particle representation contains 16000 scalar particle-state entries per population before encoding. The choice of 1000 particles fixes the tensor size used in these experiments, but it is not an architectural constraint: the AeroMELD encoder is a weighted aggregation over particles and can be evaluated on particle populations with different cardinalities. Unlike conventional bulk or modal representations, this dataset resolves the full evolution of aerosol mixing state, providing a uniquely stringent test for reduced aerosol representations. The model input is the weighted particle population itself, without prior binning by size or aggregation across particles by species. Binned mass and number distributions are computed only downstream as diagnostic targets.

The 15 chemical species tracked in the model are: Sulfate (SO4), Nitrate (NO3), Chloride (Cl), Ammonium (NH4), Sodium (Na), Dust, Black Carbon (BC), Water (H2O), Primary Organic Aerosol (POA), Marine Organic Compounds (MOC), and five lumped precursors for Secondary Organic Aerosol (SOA): high-yield aromatics (ARO1), low-yield aromatics (ARO2), long-chain alkanes (ALK1), olefins (OLE1), and alpha-pinene (API1). These species encompass primary emissions including dust, POA and BC, and secondary aerosols formed from both inorganic and organic gas-phase precursors.

\subsection{Aerosol Diagnostic Variables}

\begin{figure}
	\centering
	\includegraphics[width=0.98\linewidth]{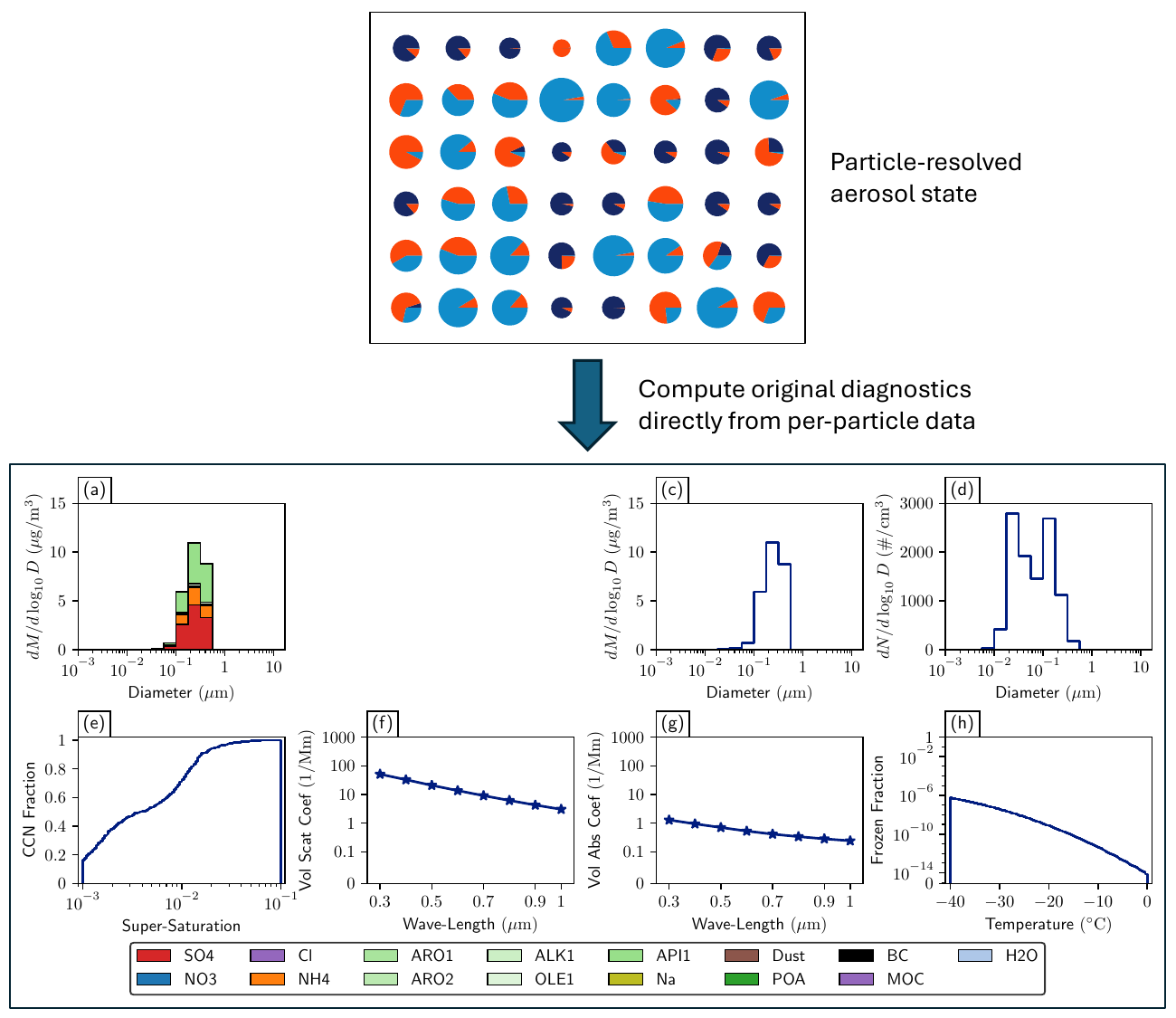}
	\caption{The aerosol diagnostics that we use in this paper, illustrated for a single population.
	(a)~The speciated mass distribution.
	(b)~Omitted for consistency with the layout of Figure~\ref{fig:08anecdiag}.
	(c)~The total mass distribution. 
	(d)~The number distribution. 
	(e)~The CCN spectrum (i.e., the cloud condensation nuclei fraction of the particles at each supersaturation level).
	(f)~The volume scattering coefficient spectrum. 
	(g)~The volume absorption coefficient spectrum.
	(h)~The frozen fraction spectrum. The diagnostics are computed from the weighted particle population and are used as diagnostic targets only; the encoder input remains the particle-resolved state shown at top.}
	\label{fig:02diagintro}
\end{figure}

There are six diagnostic variables studied in this paper, as illustrated in Figure~\ref{fig:02diagintro}, all computed from the underlying weighted particle populations. These diagnostics provide training and evaluation targets for the latent diagnostic maps; they should not be interpreted as the aerosol state representation used by AeroMELD.
The first two are the speciated mass distribution and the number distribution, which capture the size-resolved structure of the population (but ignore mixing state). The remaining four are climate-relevant and consist of CCN spectra, volume absorption and scattering coefficient spectra, and immersion freezing ice nuclei spectra. CCN spectra provide an integrated measure of aerosol size and composition, directly linking particle properties to their ability to form cloud droplets. Because droplet activation is a threshold process that depends on both size and hygroscopicity, CCN spectra serve as a robust benchmark for testing whether compressed representations retain the information most relevant for warm cloud formation. Aerosol scattering and absorption coefficients are central to direct radiative forcing and depend sensitively on mixing state, especially for black carbon and dust. By including optical diagnostics, we directly test whether the latent representations can preserve compositionally dependent absorption and scattering, which are critical for constraining aerosol-radiation interactions. Immersion freezing diagnostics provide a stringent test because ice nucleation is often controlled by trace components such as dust and soot. Small reconstruction errors in these components can lead to large differences in frozen fraction spectra. Including this diagnostic therefore probes the limits of compression methods in capturing the rare, nonlinear processes most important for mixed-phase and cirrus cloud formation. The following briefly describes the methods used to calculate these diagnostics.

\textbf{Speciated Mass Distribution:} For each aerosol population, the
speciated mass distribution provides the size-resolved mass of each
chemical component across the $B$ diameter bins. For every species $a$
and size bin $b$, we compute the mass concentration $m_{a,b}$ by
summing the masses of all particles whose diameters fall within that
bin. This diagnostic retains the detailed size stratification of each
chemical constituent, but by construction it does not capture mixing
state because the masses of different species are aggregated at the
bin level. The speciated mass distribution is one of the two
``vector'' diagnostics used to assess reconstruction fidelity of the
latent representation. For ease of visual comparison, we additionally
plot the total mass distribution obtained by summing $m_{a,b}$ across
all species. An example is shown in Figure~\ref{fig:02diagintro}(a).

\textbf{Number Distribution:} The number distribution $n_b$ describes
the size-resolved particle number concentration across the same set of
logarithmically spaced diameter bins. For each bin, we count the
number of particles whose diameters fall within that interval and
normalize by the simulation volume, yielding a discrete approximation
of $dN/d\log D$. Like the speciated mass distribution, the number
distribution reflects the overall size structure of the population but
does not carry any information about mixing state. Together, the
speciated mass and number distributions constitute the fundamental
size-resolved descriptors of each aerosol population and serve as
baseline diagnostics for evaluating reconstruction accuracy. An
example is shown in Figure~\ref{fig:02diagintro}(d).

\textbf{CCN Spectrum:} We computed the CCN fraction following
\citet{riemer2010}. For each computational particle, the Köhler
equation was solved to determine the critical supersaturation based on
the particle's dry size and chemical composition. At each prescribed
environmental supersaturation, particles with critical supersaturations
below this threshold were counted as activated using their number
weights, and the resulting number-weighted activated fraction was used
to construct the CCN spectrum. Figure~\ref{fig:02diagintro}(e) shows
the CCN fraction spectrum for the population in
Figure~\ref{fig:02diagintro}(a). The CCN diagnostic serves as a
stringent test of whether the latent
representation preserves the information most relevant for warm-cloud
activation.

\textbf{Volume Absorption and Scattering Coefficient Spectrum:} We
computed the volume absorption $\beta_{\rm a}$ and scattering
coefficients $\beta_{\rm s}$ using Mie theory, adapting the treatment
of~\citet{yao2022quantifying}. For bins without dust or black carbon, homogeneous
spheres were assumed, with refractive indices determined from
composition using volume mixing rules. For dust- and BC-containing
bins, a core–shell configuration was assumed, with the absorbing or
refractory material treated as the core and the remaining components
as the shell. Ensemble optical coefficients were then obtained by
summing the bin contributions across the distribution. This treatment
captures the influence of both size and composition on aerosol optical
behavior. Figure~\ref{fig:02diagintro}(g) shows the $\beta_{\rm a}$
spectrum for the population in Figure~\ref{fig:02diagintro}(a), and
Figure~\ref{fig:02diagintro}(f) shows the $\beta_{\rm s}$ spectrum for
the population in Figure~\ref{fig:02diagintro}(a). These diagnostics
are particularly sensitive to both size and composition, making them
an important test of whether the latent space preserves
aerosol–radiation interactions.

\textbf{Frozen Fraction Spectrum:} Ice nucleation properties were
evaluated using the ice nucleation active site (INAS) density
parameterization~\citep{Hoose2012}. For each bin containing an
ice-active component (i.e., dust~\citep{niemand2012particle} or black
carbon~\citep{Schill2020}), the number of active sites was determined
as a function of particle surface area and temperature. The
probability of freezing for each bin was then computed from the
product of its surface area and the parameterized active site
density. By aggregating over the full distribution, we obtained frozen
fraction spectra that represent the immersion freezing behavior of the
ensemble. Figure~\ref{fig:02diagintro}(h) shows the frozen fraction
spectrum for the population in
Figure~\ref{fig:02diagintro}(a). Immersion freezing is extremely
sensitive to trace ice-active components, so small reconstruction
differences can lead to large deviations in this diagnostic, making it
a rigorous test of the limits of latent compression.

\subsection{Preprocessing and Post-processing Transformations}
\label{sec:preproc}

Figure~\ref{fig:03schematic} summarizes the transformations used to
condition the particle-resolved inputs and diagnostic outputs before
they are passed through the neural networks. For each aerosol
population $\pi$, we first separate the total number concentration
$n$ from the normalized population $\hat{\pi}=\{(p_i,\mu_i)\}$, where
$p_i=n_i/n$. The scale $n$ is carried explicitly around
the learned shape encoder, while the neural network acts on the
normalized population.

The encoder-side transformation $\TT_{\rm enc}$ is applied
particle-by-particle to the normalized population. Each particle
composition vector is shifted by a small numerical floor
$\epsilon_\mu$ and standardized with an affine transformation $\Z$
whose centering and scaling constants are estimated from the training
data. The normalized particle weight $p_i$ is retained and used in the
weighted set aggregation. Thus $\TT_{\rm enc}$ produces a numerically
well-conditioned weighted set $u$ while preserving the permutation
invariance and linear weighting structure required by the AeroMELD
encoder. The transformed set is then passed to the variational shape
encoder $\hat{\Phi}_\theta$ defined in
Eq.~(\ref{eqn:encoder-form-shape}), whose weighted aggregation over
the per-particle encoder $\phi_\theta$ produces the deterministic
shape coordinate used as the VAE mean. During training, a separate
auxiliary variance head supplies the covariance of the latent shape
distribution. The scale component $n$ is not sampled or regularized;
it bypasses the VAE and is reintroduced only when diagnostics with
extensive units are reconstructed.

Analogously, the diagnostic targets are transformed before being
compared with the decoder output. We write the scale-normalized
diagnostic vector as $\hat{y}$, where extensive diagnostics such as
mass, number, and optical coefficients are divided by $n$, while
scale-invariant fractional diagnostics such as CCN and frozen fraction
are already number-normalized. The decoder-side transformation
$\TT_{\rm dec}$ separates
the mass-related part of $\hat{y}$ into a total mass magnitude and a
composition-fraction component. Specifically, after adding a small
floor $\epsilon_m$, the total mass block is obtained with an
$\ell_1$-norm and the remaining mass vector is divided by this total
to obtain a normalized composition vector. These magnitude and
composition components are then power transformed and standardized
separately. The remaining diagnostic block is shifted by
$\epsilon_h$, power transformed, and standardized in the same manner.
This decomposition reduces the dynamic range of the learning problem,
keeps near-zero bins numerically stable, and separates total loading
from relative composition.

The learned latent diagnostic model predicts the transformed variable
$v$. To recover physical diagnostics, we apply the inverse
transformation $\TT_{\rm dec}^{-1}$ and then reapply the scale $n$ to
the extensive components. All standardization constants and
transformation parameters are fit using the training data only and are
held fixed for validation and test cases. The reconstruction errors
reported below are computed after these inverse transformations, in
the physical diagnostic variables shown in
Figure~\ref{fig:02diagintro}.

\begin{figure}
	\centering
	\includegraphics[width=0.98\linewidth]{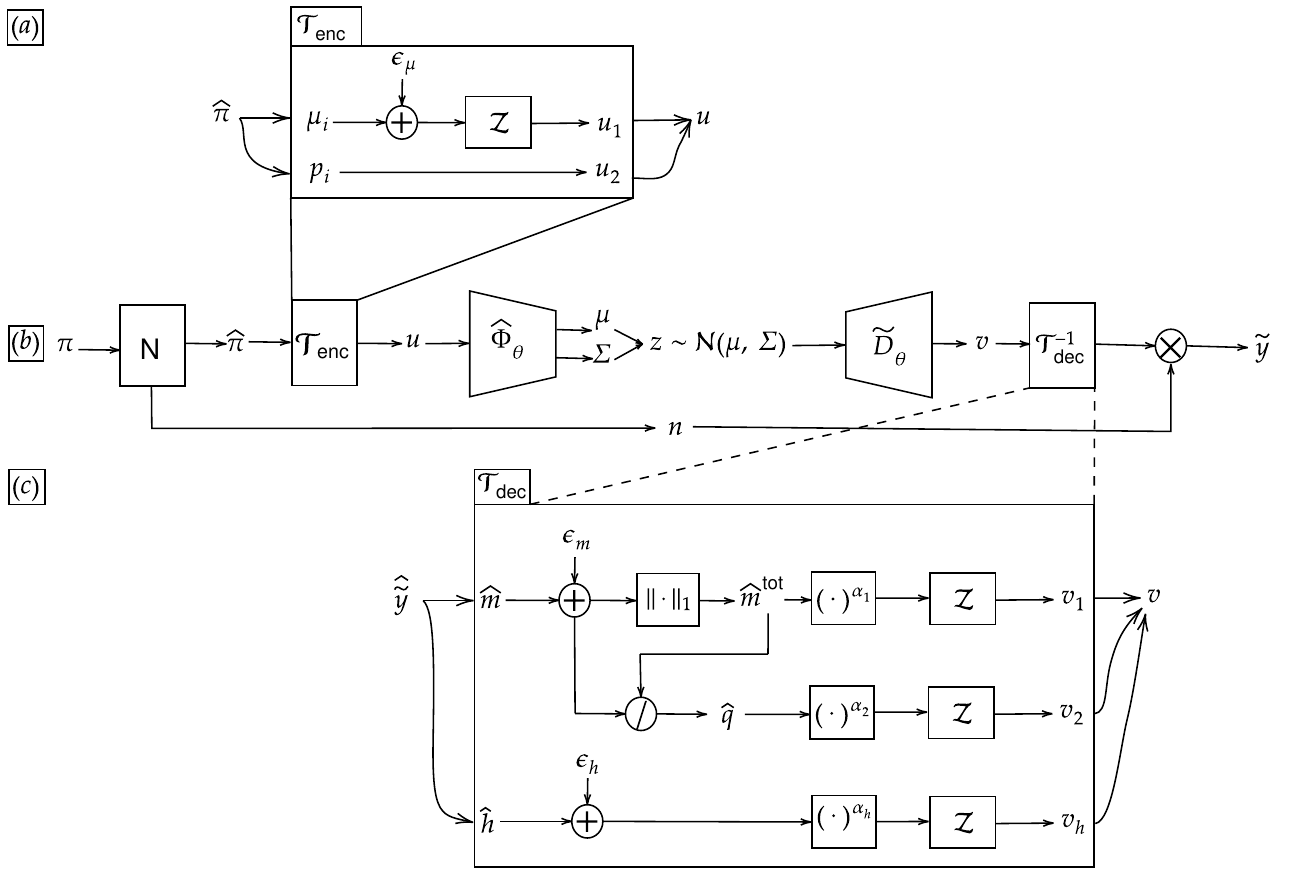}\vspace{-3mm}
	\caption{Preprocessing and post-processing transformations used in the AeroMELD diagnostic pipeline.
	(a) The encoder-side transformation $\TT_{\rm enc}$ maps each normalized particle composition and weight to standardized set features.
	(b) The full variational diagnostic pipeline separates an aerosol population into its total number concentration $n$ and normalized population $\hat{\pi}$, encodes the normalized population into the latent shape variable $z$, predicts transformed diagnostics $v$, and maps them back to physical units.
	(c) The decoder-side transformation $\TT_{\rm dec}$ separates scale-normalized diagnostic targets into total-magnitude, composition-fraction, and remaining diagnostic blocks before applying power transformations and standardization.}\label{fig:03schematic}\vspace{-12mm}
\end{figure}

\subsection{Train and Test Split}

To ensure that our model generalizes to unseen aerosol evolutionary pathways, we partitioned the dataset by randomly splitting entire scenarios into training and testing sets, rather than splitting individual samples. This strategy prevents data leakage from temporally correlated samples within the same scenario. We employed an 80--20 train-test split, assigning 80\% of the scenarios to the training set and 20\% to the test set. Because all 25 time snapshots from each selected scenario are retained, each randomized split contains 20000 training populations and 5000 testing populations. To ensure the robustness of our findings, this process was repeated with 10 different randomization seeds. A separate model was trained for each seed, and all statistics reported in this paper were averaged across these 10 randomized runs.

\subsection{VAE Regularization}

To encourage a smooth and well-structured latent space, we employ variational autoencoder (VAE) regularization~\citep{kingma2013auto,rezende2014stochastic}. The deterministic shape encoder $\hat{\Phi}_\theta$ from Theorem~\ref{thm:encoder-form} is used as the mean map of a Gaussian recognition distribution over $z$, while a separate auxiliary variance head supplies a diagonal covariance. That is, for $d = L - 1$,
\begin{equation}
	\label{eq:vae_encoder}
\begin{aligned}
	\mu_\theta(\hat{\pi}) &= \hat{\Phi}_\theta(\hat{\pi}), \\
	\Sigma_\theta(\hat{\pi}) &=
		\text{diag}\bigl(\sigma_{\theta,1}^2(\hat{\pi}), \ldots,
		\sigma_{\theta,d}^2(\hat{\pi})\bigr), \\
	q_\theta(z \mid \hat{\pi}) &=
		\mathcal{N}\bigl(\mu_\theta(\hat{\pi}), \Sigma_\theta(\hat{\pi})\bigr).
\end{aligned}
\end{equation}
The covariance $\Sigma_\theta$ is used only for VAE sampling and
regularization; the latent shape coordinate used for deterministic
latent operations is the mean $\mu_\theta(\hat{\pi})$. The scale
component $n$ remains deterministic. During training, latent samples
are drawn from this distribution,
\begin{equation}
	\label{eq:vae_sample}
	z \sim q_\theta(z \mid \hat{\pi}),
\end{equation}
using the reparameterization trick to allow gradient propagation. We impose a standard normal prior $p(z) = \mathcal{N}(0, I)$ and add a mean-reduced Kullback--Leibler (KL) divergence penalty to the reconstruction loss~(\ref{eq:loss_reconstruction}). The total training loss is
\begin{equation}
	\label{eq:loss_total}
	\mathcal{L}_{\rm total}(\theta) = \mathcal{L}_{\rm reconstruct}(\theta)
	+ \frac{w_\mu}{2d} \mathbb{E}_\pi\bigl[\|\mu_\theta(\hat{\pi})\|_2^2\bigr]
	+ \frac{w_\sigma}{2d} \sum_{i=1}^{d} \mathbb{E}_\pi\bigl[\sigma_{\theta,i}^2(\hat{\pi}) - 1 - \log \sigma_{\theta,i}^2(\hat{\pi})\bigr],
\end{equation}
where $w_\mu$ and $w_\sigma$ are hyperparameters controlling the regularization strength of the mean and variance terms, respectively. This VAE regularization encourages the encoder to produce latent representations that are spread smoothly across the latent space, which improves generalization and enables meaningful interpolation between aerosol populations.

\subsection{Mixup Regularization}

To further encourage the latent space to support meaningful interpolation, we employ mixup data augmentation~\citep{zhang2018mixup}. During training, we augment the dataset by sampling pairs of aerosol populations $(\pi_A, \pi_B)$ and encoding them to obtain noisy latent representations $(n_A, z_A)$ and $(n_B, z_B)$. We then sample an interpolation coefficient $\gamma \sim \text{Beta}(1, 1)$, which is equivalent to a uniform distribution on $[0, 1]$. The interpolated latent representation is formed by convex interpolation in the scale and shape coordinates,
\begin{align}
	\label{eq:mixup_latent}
	n_{\rm mix} &= \gamma n_A + (1 - \gamma) n_B, \\
	z_{\rm mix} &= \gamma z_A + (1 - \gamma) z_B,
\end{align}
where $\gamma$ is interpreted as a shape-space interpolation weight rather than as a coefficient in an unnormalized physical population mixture. The corresponding target is the same convex interpolation of the scale-normalized diagnostic variables used for training:
\begin{equation}
	\label{eq:mixup_diag}
	\hat{y}_{\rm mix} = \gamma \hat{y}_A + (1 - \gamma) \hat{y}_B,
\end{equation}
where $\hat{y}_A$ and $\hat{y}_B$ are the transformed diagnostic targets for $\pi_A$ and $\pi_B$, respectively. These interpolated samples $(n_{\rm mix}, z_{\rm mix}, \hat{y}_{\rm mix})$ are added to the training set, and the standard reconstruction loss~(\ref{eq:loss_reconstruction}) is computed over the augmented data. This data augmentation encourages the learned latent diagnostics to vary smoothly across the latent shape space, improving generalization to intermediate aerosol states.

\section{Results}
\label{sec:results}

This section presents the results of using the AeroMELD framework for diagnostic reconstruction. We first evaluate the model's ability to reconstruct aerosol diagnostics by examining both individual examples and collective error metrics across the test dataset. We then analyze the structure of the learned latent space to assess how the model organizes the aerosol data.

\subsection{Aerosol Diagnostic Reconstruction Examples}

Figure \ref{fig:08anecdiag} shows an illustrative example comparing
the original aerosol diagnostics (blue) with those reconstructed from
the AeroMELD latent representation (red). The speciated mass
distribution and number distribution (panels a–d) are captured with
high fidelity: the reconstructed modes match the original in both
location and magnitude, with only small discrepancies in bins with
very low mass or number concentration. The climate-relevant
diagnostics (panels e–h) likewise show excellent agreement. The CCN
spectrum closely matches the original activation curve across the full
range of supersaturations. The absorption and scattering coefficient
spectra closely follow the originals across wavelengths, indicating
that AeroMELD retains the composition-dependent optical signatures
associated with absorbing and scattering species. The frozen-fraction
spectrum also aligns well, even though immersion freezing is highly
sensitive to trace dust and BC. Panels (i$1$)--(i${15}$) and
(j$1$)--(j${15}$) further decompose the speciated mass distributions by
individual species. The common vertical scale in the (i) panels
emphasizes the dominant components, whereas the independently scaled
(j) panels deliberately magnify differences in trace species that are
visually suppressed on the common scale. For ARO2, Na, MOC, and H2O,
the absolute reconstructed mass is small relative to the dominant mass
components, but the relative errors can be appreciable. In particular,
the reconstruction assigns small positive amounts of Na, MOC, and H2O
where the reference values are zero or nearly zero, overestimates
portions of the POA and BC distributions, and underestimates dust. Thus,
the model captures the dominant mass structure and associated size
ranges but does not enforce exact sparsity or exact species presence
and absence. Taken together, this example illustrates that a
low-dimensional AeroMELD latent variable reconstructs a wide range of
aerosol diagnostics---including threshold processes and
composition-specific features---with high accuracy, while trace-species
reconstruction remains more challenging.

\begin{figure}
	\centering
	\includegraphics[page=19,width=0.8\linewidth]{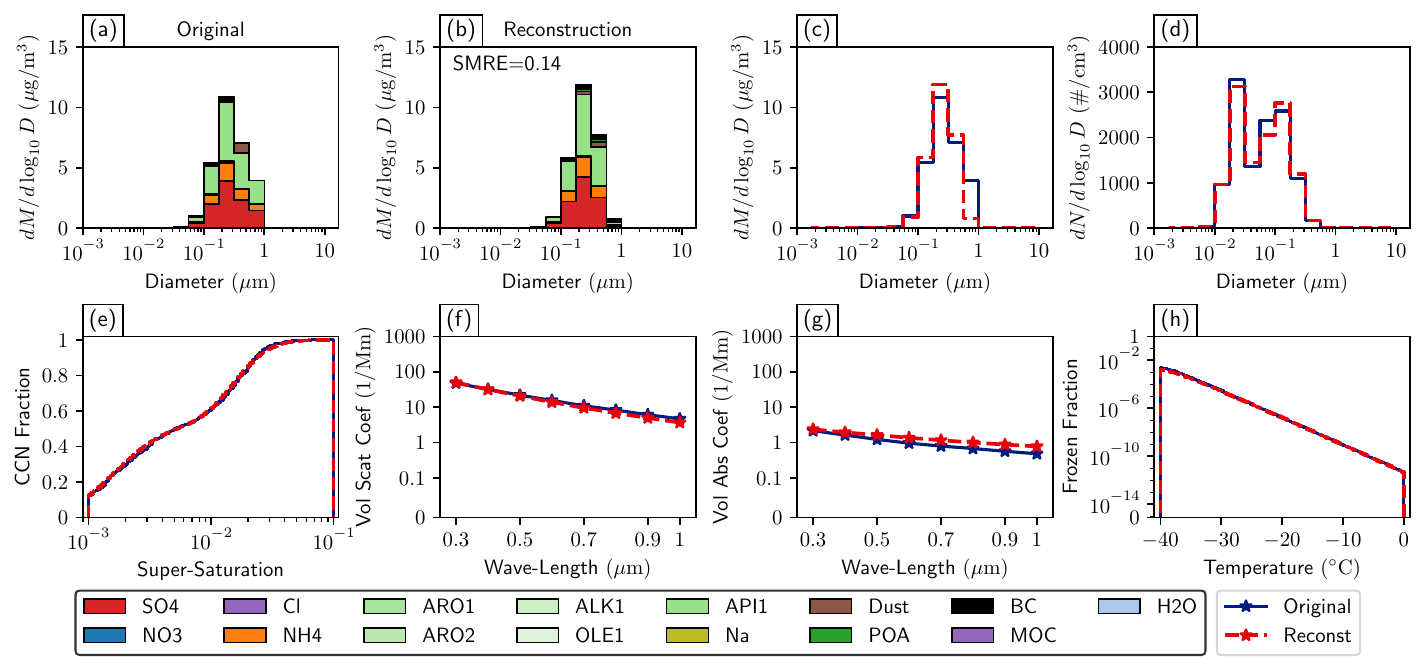}
	\includegraphics[page=20,width=0.8\linewidth]{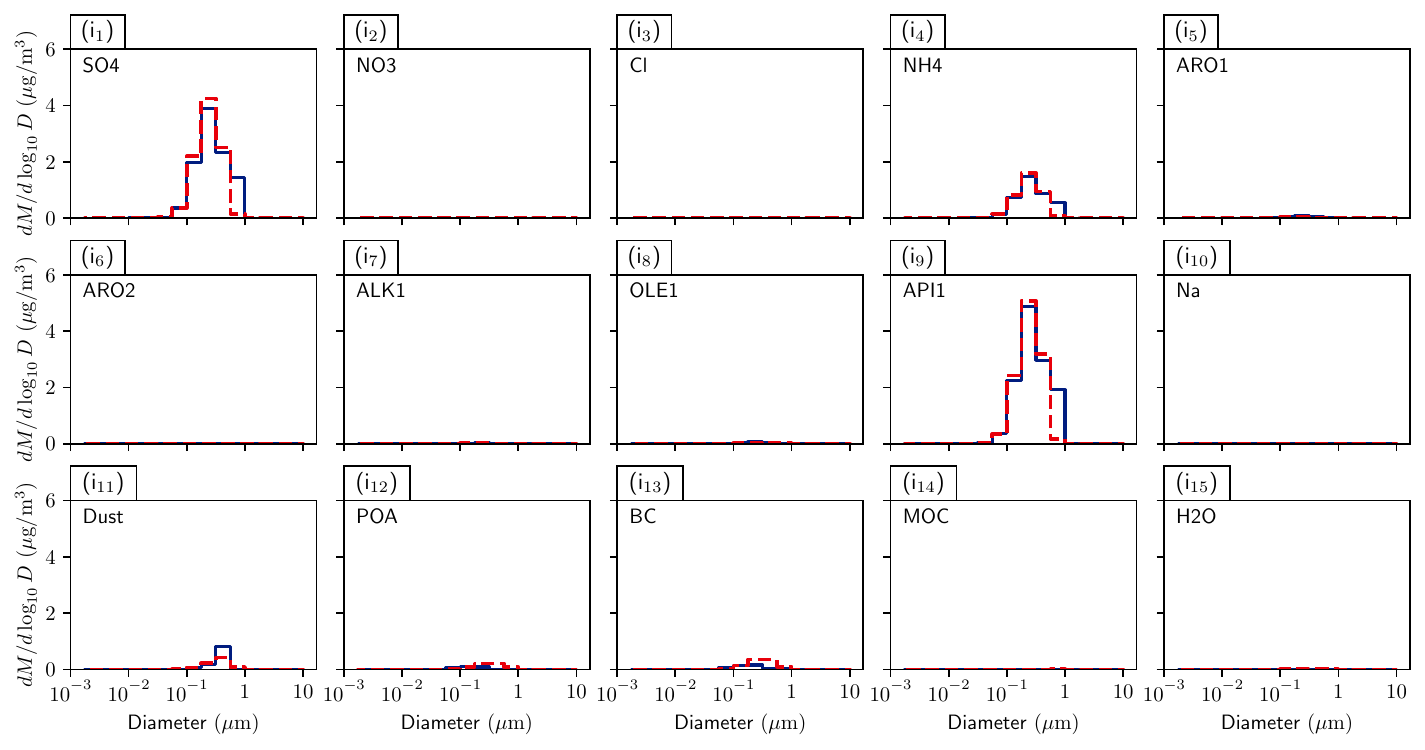}
	\includegraphics[page=21,width=0.82\linewidth]{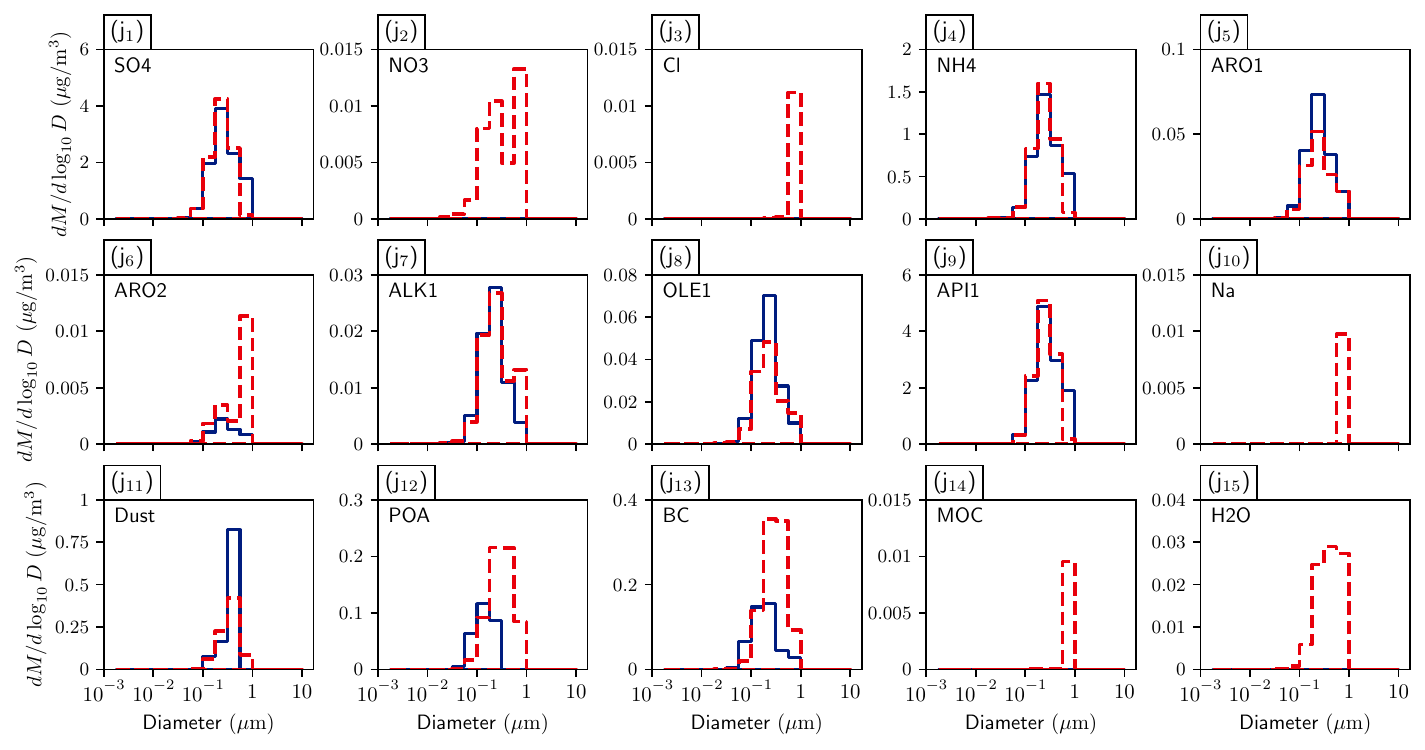}
	\vspace{-3mm}\caption{The aerosol diagnostics for the original and reconstructed variants of a particular test sample. The (a)--(h) plots show the speciated mass and number distributions, the CCN spectrum, the volume scattering and absorption coefficient curves, and the frozen fraction spectrum. The (i$_1$)--(i$_{15}$) plots show the mass distributions conditioned for each chemical. The (j$_1$)--(j$_{15}$) plots show the same mass distributions, with the vertical axis being independently scaled. The particular sample in the figure has a speciated mass relative error of 0.14. Binned distributions in panels (a)–(d) and (i)–(j) are diagnostic projections of the particle-resolved state, not the input representation.}\label{fig:08anecdiag}
\end{figure}

\subsection{Collective Aerosol Diagnostic Summaries}

Figures \ref{fig:09smrytest} and \ref{fig:09mnsmrytest} summarize the
reconstruction performance across the test dataset. For the
climate-relevant diagnostics (Figure \ref{fig:09smrytest}), errors are
generally small: CCN spectra, optical coefficients, and
frozen-fraction curves all show strongly right-skewed error
distributions, with most samples clustering at low error. The
reconstructed diagnostics closely follow the 1:1 line, indicating that
the latent representation captures both the overall shape and
magnitude of these quantities across several orders of
magnitude. Notably, even the frozen-fraction spectra—highly sensitive
to trace dust and BC—are reproduced with good fidelity.  Figure
\ref{fig:09mnsmrytest} shows the corresponding error distributions for
the speciated mass, total mass, bulk mass, and number
distributions. These diagnostics also exhibit low typical errors with
right-skewed distributions, reflecting consistently accurate
reconstruction of the size-resolved structure of the aerosol
population. Number-distribution errors are smallest, while mass-based
diagnostics show somewhat larger variation, primarily for cases with
very low loadings. Overall, the results demonstrate that the AeroMELD
latent space provides reliable reconstruction across all diagnostic
types. Quantitative summary statistics are given in Table
\ref{tab:rcnsterrs}.

\begin{figure}
	\centering
	\includegraphics[width=0.9\linewidth]{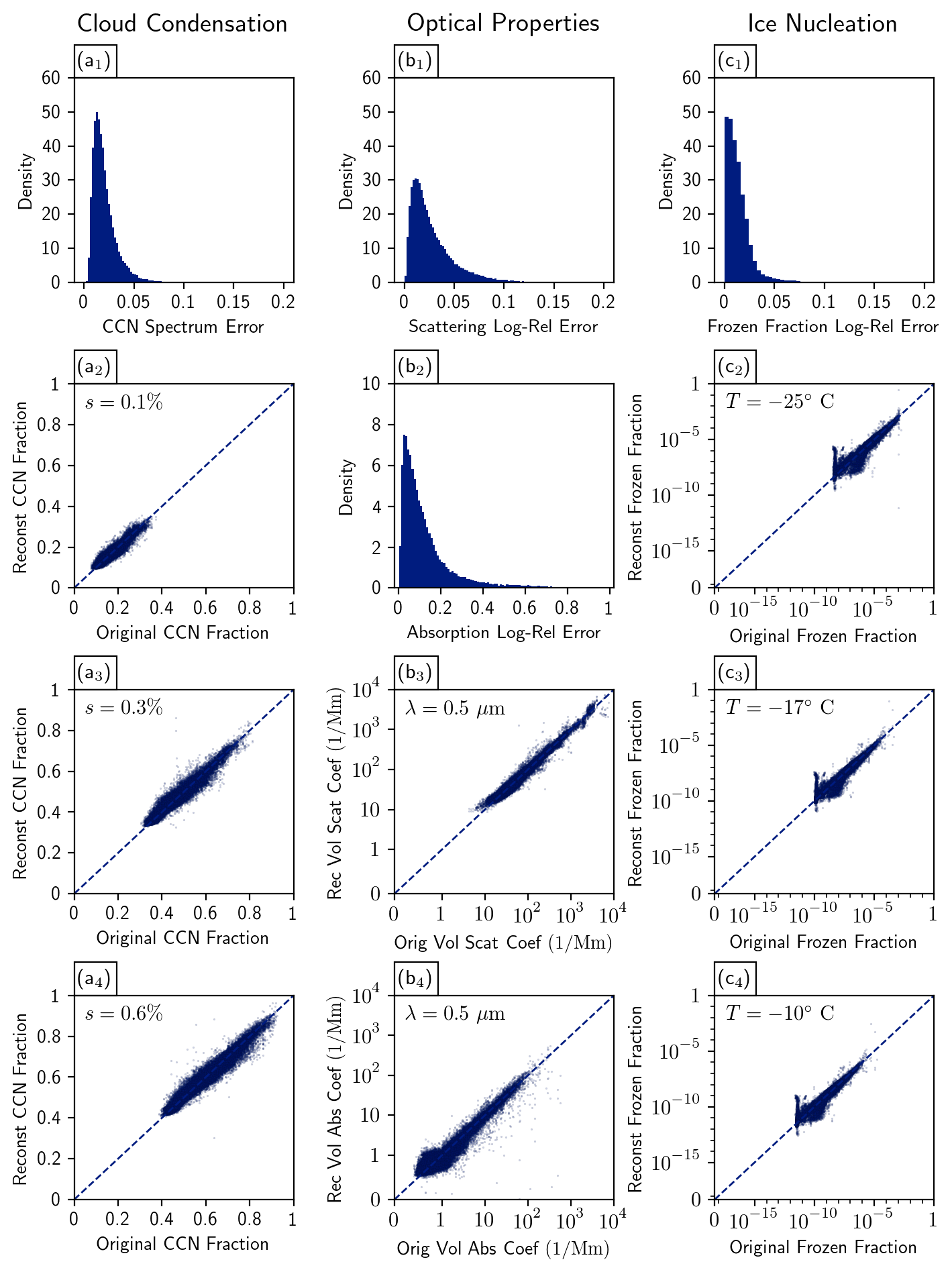}
	\caption{The collective aerosol diagnostic summary plots on the test data. 
	The left column shows (1) the histogram of the CCN spectrum error, and (2) the reconstructed vs. original CCN fractions at three different supersaturation levels ($s=0.1\%$, $0.3\%$, and $0.6\%$).
	The middle column shows the scattering and absorption error histograms, and 
	the corresponding reconstructed vs. original scatter plots at a wavelength of $\lambda=0.5 \, {\rm \mu m}$. The right column shows the frozen fraction error histogram, and the reconstructed vs. original scatter plots at three different temperatures of $T=-25 \, {\rm ^{\circ} C}$, $-17 \, {\rm ^{\circ} C}$, and $-10 \, {\rm ^{\circ} C}$.}\label{fig:09smrytest}
\end{figure}

\begin{figure}
	\centering
	\includegraphics[width=0.9\linewidth]{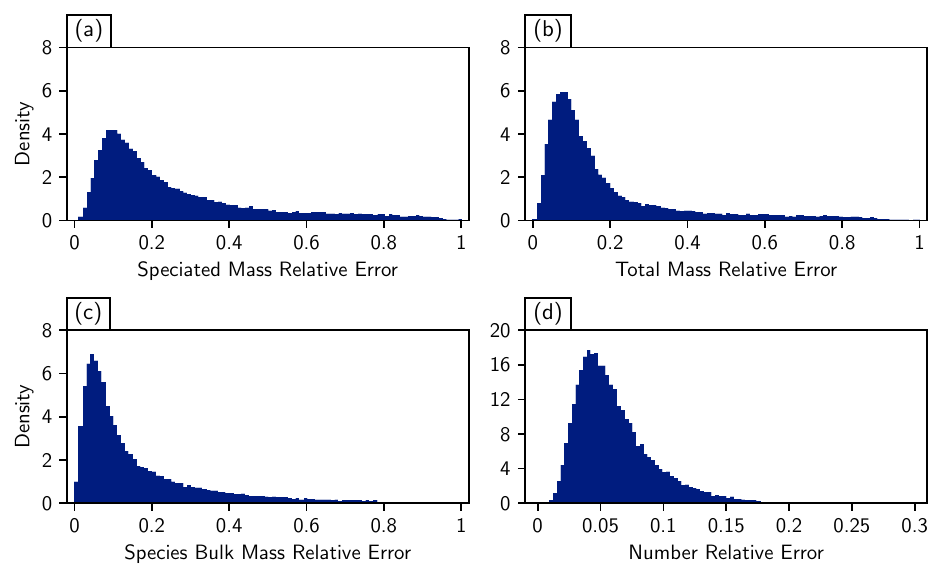}
	\caption{The relative error histogram of the speciated, total, and bulk mass and number distributions for the testing dataset.}\label{fig:09mnsmrytest}
\end{figure}

% The High KL model reconstruction errors
\renewcommand{\arraystretch}{1}
\newcommand{\ZR}{\phantom{0}}
\begin{table}
	\centering	
	\begin{tabular}{p{0.45\textwidth}p{0.08\textwidth}p{0.25\textwidth}}
	\midrule
	Aerosol Diagnostic Metric & \ZR Mean & [95\% CI] \\ \midrule
	CCN Spectrum Relative Error & \ZR2.04\% & [\ZR1.92\%,\ZR2.20\%] \\\hline
	Scattering Log-Rel Error & \ZR2.72\% & [\ZR2.6\%,\ZR2.85\%] \\\hline
	Absorption Log-Rel Error & 12.8\ZR\% & [11.9\ZR\%,13.5\ZR\%] \\\hline
	Frozen Fraction Log-Rel Error & \ZR1.37\% & [\ZR1.29\%,\ZR1.46\%] \\\hline
	Speciated Mass Relative Error & 25.8\ZR\% & [25.1\ZR\%,26.4\ZR\%] \\\hline
	Number Relative Error & \ZR6.20\% & [\ZR5.98\%,\ZR6.50\%] \\\hline
	Total Mass Relative Error & 19.3\ZR\% & [18.8\ZR\%,19.8\ZR\%] \\\hline
	Species Bulk Mass Relative Error & 16.6\ZR\% & [16.0\ZR\%,17.2\ZR\%] \\\hline
	\hline
	\end{tabular}
	\caption{The average errors of the trained model on the held-out test set.}
	\vspace{-5mm}
	\label{tab:rcnsterrs}
\end{table}

\subsection{Low-Dimensional Visualizations}

Figure~\ref{fig:10mdscnn} shows two-dimensional visualizations of the
learned latent shape space and of the speciated mass diagnostics. The
top row uses t-SNE to emphasize local neighborhood structure in the
latent variables, while the middle row shows the corresponding PCA
projection. In both projections, the test samples overlap the
training samples rather than forming separated clusters, indicating
that held-out aerosol populations are mapped onto the same learned
manifold as the training data. Unlike the training and test samples,
the generated cases do not begin with a particle-resolved aerosol
population. Instead, their latent shape coordinates are sampled
directly from the Gaussian prior and passed through the learned latent
diagnostic model, without first encoding a physical-space population.
The generated samples cover many of the same neighborhoods as the
embedded data in the t-SNE projection, while their PCA projection is
concentrated near the center because their latent shape coordinates
are sampled from the Gaussian prior. These projections therefore
provide qualitative evidence that the regularized latent space can
generate diagnostic combinations resembling those represented in the
dataset.

The bottom row of Figure~\ref{fig:10mdscnn} projects the
high-dimensional speciated mass distributions using t-SNE. For both
training and test data, the reconstructed diagnostics are interleaved
with the corresponding original diagnostics, showing that the
reconstruction preserves neighborhood relationships among aerosol
populations in diagnostic space. The generated diagnostic outputs also
occupy many of the same neighborhoods as the original and reconstructed
samples. Because the present model decodes diagnostics rather than a
particle-resolved population, this overlap does not by itself establish
that every generated diagnostic vector corresponds to a jointly
realizable aerosol population. Instead, it provides evidence of
generative coverage and latent-space regularity. The physical
interpretation of the latent state comes from its barycentric mixing
behavior and its representation as learned test-function moments.
Together, the latent-space and diagnostic-space projections indicate
that AeroMELD learns a compact representation that generalizes to
unseen scenarios while preserving the dominant structure of the
speciated aerosol mass distributions.

\begin{figure}
	\centering
	\includegraphics[width=0.98\linewidth]{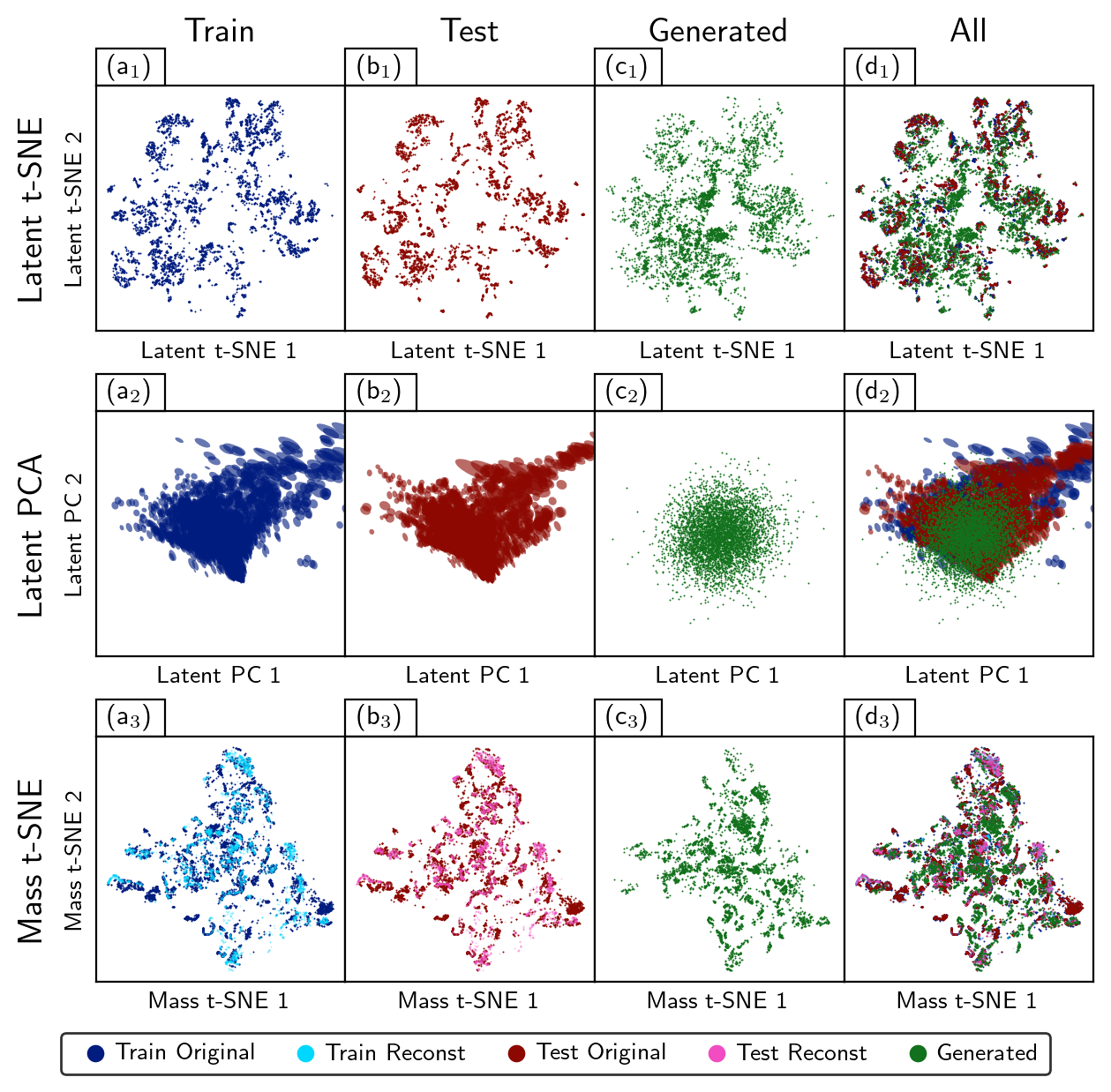}
	\caption{Low-dimensional visualizations of the samples. The \textit{top and middle rows} show the 2D t-SNE and PCA representation of the samples in the latent space. The \textit{bottom row} shows the 2D t-SNE representation of the speciated mass distributions. The \textit{left, middle-left, and middle-right columns} are dedicated to the training, testing, and generated samples, respectively. Generated samples begin from latent shape coordinates sampled from the Gaussian prior, rather than from particle-resolved populations, and are passed through the latent diagnostic model to obtain their diagnostic outputs. The \textit{right column} is a compilation of all three groups in one plot. The horizontal and vertical axes are shared in each row.}\label{fig:10mdscnn}
\end{figure}

\subsection{Hyperparameter Sensitivity}

Figure~\ref{fig:10abltn} summarizes the sensitivity of the model to
three hyperparameters using the geometric average of the diagnostic
errors listed in Table~\ref{tab:rcnsterrs}. The strongest dependence
is on the latent dimension. Very small latent spaces underfit the
particle-resolved distributions, while the average diagnostic error
decreases rapidly as the dimension is increased from 2 to 10. Beyond
this range, the improvement is comparatively small, indicating that the
main configuration lies near the point of diminishing returns for this
dataset.

The KL-weight sweep shows the expected tradeoff between accurate
reconstruction and regularization of the latent distribution. Weak to
moderate KL regularization preserves the aerosol diagnostics most
accurately, whereas stronger KL weights increase the diagnostic error
by forcing the encoded populations closer to the Gaussian prior. The
mixup weight has a weaker effect over the range tested: the error is
relatively flat, with a modest improvement at larger weights. This
suggests that mixup improves the smoothness of the learned
shape-space interpolation without imposing a substantial reconstruction
penalty. The boxed settings in Figure~\ref{fig:10abltn} indicate the
configuration used for the main reconstruction and visualization
results.

\begin{figure}
	\centering
	\includegraphics[page=1,width=0.98\linewidth]{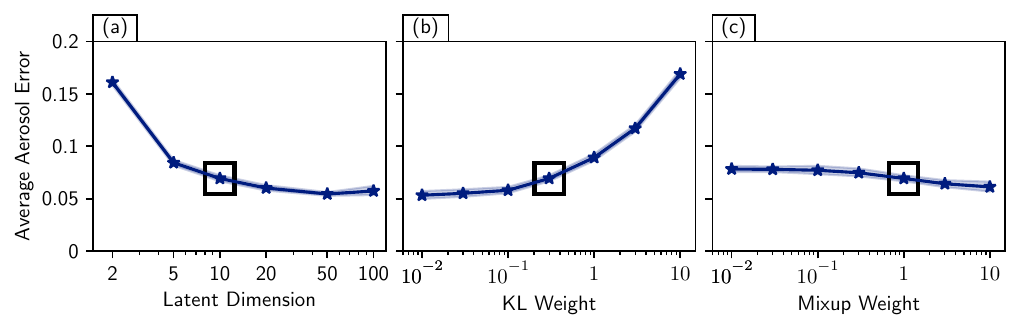}
	\vspace{-3mm}\caption{Ablating the average aerosol diagnostic metrics against (a) the latent dimensionality, (b) the KL weight, and (c) the mixup regularization weight. The vertical axis denotes the geometric average of the aerosol diagnostic metrics detailed in Table~\ref{tab:rcnsterrs}. The boxed settings indicate the configuration used for the main results.}\label{fig:10abltn}
\end{figure}

\subsection{Interpolation of Aerosol Populations}

Figure~\ref{fig:10interp} evaluates interpolation between two held-out
aerosol populations. The endpoint populations are encoded, then
their latent coordinates and corresponding diagnostic targets are
interpolated at several values of the interpolation weight $\lambda$.
This provides a direct test of whether the linear structure of the
AeroMELD latent space produces meaningful intermediate diagnostic
states, rather than only accurate reconstructions of the endpoint
samples. The mixup regularization introduced above encourages this
behavior during training; here, the interpolation is used as a
held-out diagnostic of the learned representation.

In the representative example shown in Figure~\ref{fig:10interp}, the
speciated mass and number distributions evolve smoothly from one
endpoint to the other as $\lambda$ increases from 0 to 1. The
reconstructed diagnostic curves closely track the corresponding
interpolated targets for the CCN spectrum, optical coefficients, and
frozen fraction, including across the intermediate values of
$\lambda$. This indicates that latent interpolation produces plausible
intermediate diagnostic states and preserves the climate-relevant
diagnostic structure along the path between observed samples.

\begin{figure}
	\centering
	\includegraphics[page=7,width=0.8\linewidth]{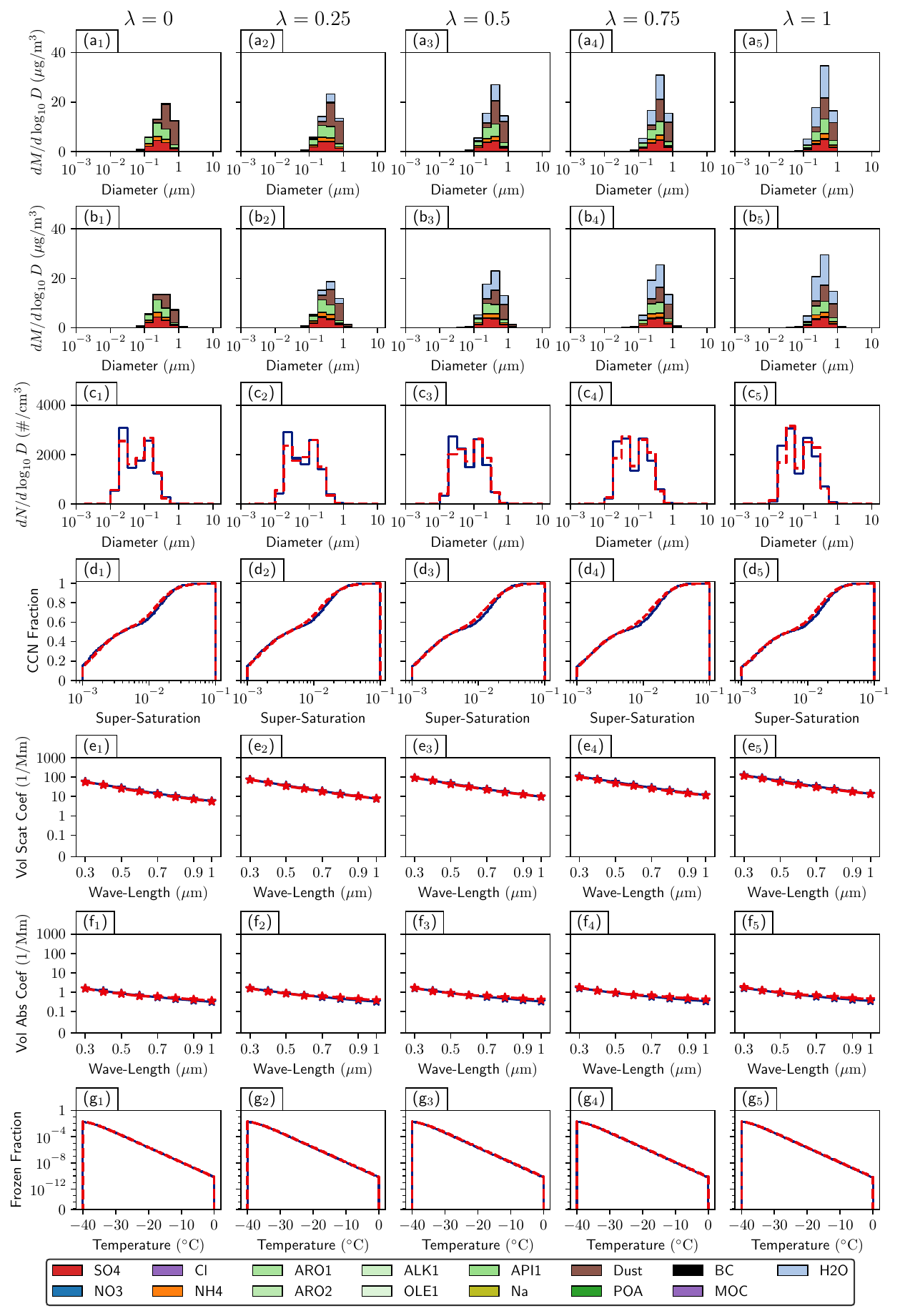}
	\vspace{-3mm}\caption{Representative interpolation between two held-out aerosol populations. Columns show interpolation weights from $\lambda=0$ to $\lambda=1$. Rows (a) and (b) show the target and reconstructed speciated mass distributions, respectively; row (c) shows the number distribution; rows (d)--(g) show the CCN spectrum, scattering coefficient, absorption coefficient, and frozen fraction. Blue curves denote the interpolated target diagnostics and red curves denote the corresponding latent reconstructions.}\label{fig:10interp}
\end{figure}

\section{Conclusions}
\label{sec:conclusions}

In this work, we introduced AeroMELD, a mathematically grounded
framework for constructing low-dimensional latent representations of
aerosol populations that preserve their intrinsic physical and
structural properties. AeroMELD enforces permutation invariance and
linearity with respect to aerosol population addition, leading to a
canonical scale–shape decomposition in which total particle number
concentration is represented explicitly and the latent shape is a
barycentric combination of per-particle embeddings. This structure
can equivalently be viewed as evaluating learned test functions
against the aerosol population measure, so the representation is not
tied to a fixed particle support or particle count. It ensures that
mixing and scalar concentration changes can be
represented exactly in latent space and that latent variables behave
like tracers under emissions and transport, which are capabilities
essential for eventual integration into three-dimensional atmospheric
models. This linearity should not be interpreted as a loss of
diagnostic expressiveness: relative to a Deep Sets diagnostic
pipeline, AeroMELD simply keeps the nonlinear post-aggregation map in the
learned diagnostic rather than in the stored latent state. In the base
experiments, this stored state has 10 dimensions, compared with 16000
scalar particle-state entries in the sampled 1000-particle
representation used as input, corresponding to a nominal 1600-fold
reduction in scalar state dimension. Although these experiments use
1000-particle samples, the same trained encoder can be applied to any
finite weighted particle population.

This paper advances beyond the first two papers in the series, which
demonstrated that variational autoencoders and conditional generative
models can efficiently represent and reconstruct binned aerosol
size–composition distributions. Those earlier models captured many
climate-relevant diagnostics but were limited by the loss of
mixing-state information and by nonlinear latent spaces in which
aerosol populations did not combine linearly. As a result, latent
variables in Papers 1 and 2 could not support physically consistent
representations of emissions or transport and were not suitable for
embedding aerosol-process operators. In contrast, AeroMELD uses
particle-resolved data and embeds the linear structure of aerosol
populations directly into the latent space, overcoming these two
central limitations.

Using particle-resolved simulations as ground truth, we demonstrated
that the AeroMELD latent space reconstructs aerosol diagnostics with
high fidelity across diverse categories. Size-resolved speciated mass
and number distributions are accurately recovered, and
climate-relevant diagnostics—including CCN spectra, optical absorption
and scattering coefficients, and immersion-freezing behavior—are
reproduced across wide dynamic ranges and in the presence of threshold
nonlinearities. These results show that a compact latent
representation derived from the underlying population semilinearity
can capture both smooth size-resolved features and sensitive,
composition-dependent processes. These diagnostic results establish
that the AeroMELD latent state retains the information needed before
learning nonlinear aerosol-process operators. Because emissions and
mixing operations are exact in latent space, only the inherently
nonlinear components of aerosol evolution—such as coagulation or
gas–particle partitioning—require learned latent operators. Future
work in this series will build on this framework to develop and
evaluate latent-space dynamical models, enabling fast, physically
consistent aerosol evolution suitable for coupling with
next-generation ML-based Earth system models.

%\clearpage
%%%%%%%%%%%%%%%%%%%%%%%%%%%%%%%%%%%%%%%%%%%%%%%
%
% DATA SECTION and ACKNOWLEDGMENTS
%
%%%%%%%%%%%%%%%%%%%%%%%%%%%%%%%%%%%%%%%%%%%%%%%

\section*{Conflict of Interest}
The authors declare no conflicts of interest relevant to this study.

\section*{Open Research Section}

The underlying data for this study can be accessed at \url{https://doi.org/10.13012/B2IDB-2774261_V1}. The code used for the analysis is available at \url{https://github.com/ehsansaleh/partnn}.

\section*{Acknowledgments}

This work used GPU resources at the Delta supercomputer of the National Center for Supercomputing Applications through Allocation CIS220111 from the Advanced Cyberinfrastructure Coordination Ecosystem: Services and Support (ACCESS) program~\citep{boerner2023access}, which is supported by National Science Foundation grants \#2138259, \#2138286, \#2138307, \#2137603, and \#2138296.

This work was supported by the U.S. Department of Energy, Office of Science, Office of Biological and Environmental Research under Award Number DE-SC0022130, and the Laboratory Directed Research and Development program at Sandia National Laboratories. Sandia National Laboratories is a multimission laboratory managed and operated by National Technology and Engineering Solutions of Sandia LLC, a wholly owned subsidiary of Honeywell International Inc. for the U.S. Department of Energy’s National Nuclear Security Administration contract DE-NA0003525.

This paper describes objective technical results and analysis. Any subjective views or opinions that might be expressed in the paper do not necessarily represent the views of the U.S. Department of Energy or the United States Government.

%%%%%%%%%%%%%%%%%%%%%%%%%%%%%%%%%%%%%%%%%%%%%%%
% Appendices
%%%%%%%%%%%%%%%%%%%%%%%%%%%%%%%%%%%%%%%%%%%%%%%
%\clearpage
\appendix
\renewcommand{\thesection}{Appendix \Alph{section}}

\section{Additional Results}

Figures~\ref{fig:a05anecdiag} and~\ref{fig:a06anecdiag} present other examples of the model's reconstruction performance on a single, specific test sample, following the same detailed layout as Figure~\ref{fig:08anecdiag}. Together, these examples provide a more comprehensive view of the model's capabilities by showcasing its performance on different cases. These particular samples have a speciated mass relative error of 0.09 and 0.17, respectively, spanning both lower and higher errors than the sample in Figure~\ref{fig:08anecdiag}. The higher-error example illustrates that even when individual species, such as dust, are reconstructed less accurately, the climate-relevant diagnostics are still generally recovered with high accuracy.

\begin{figure}
	\centering
	\includegraphics[page=7,width=0.8\linewidth]{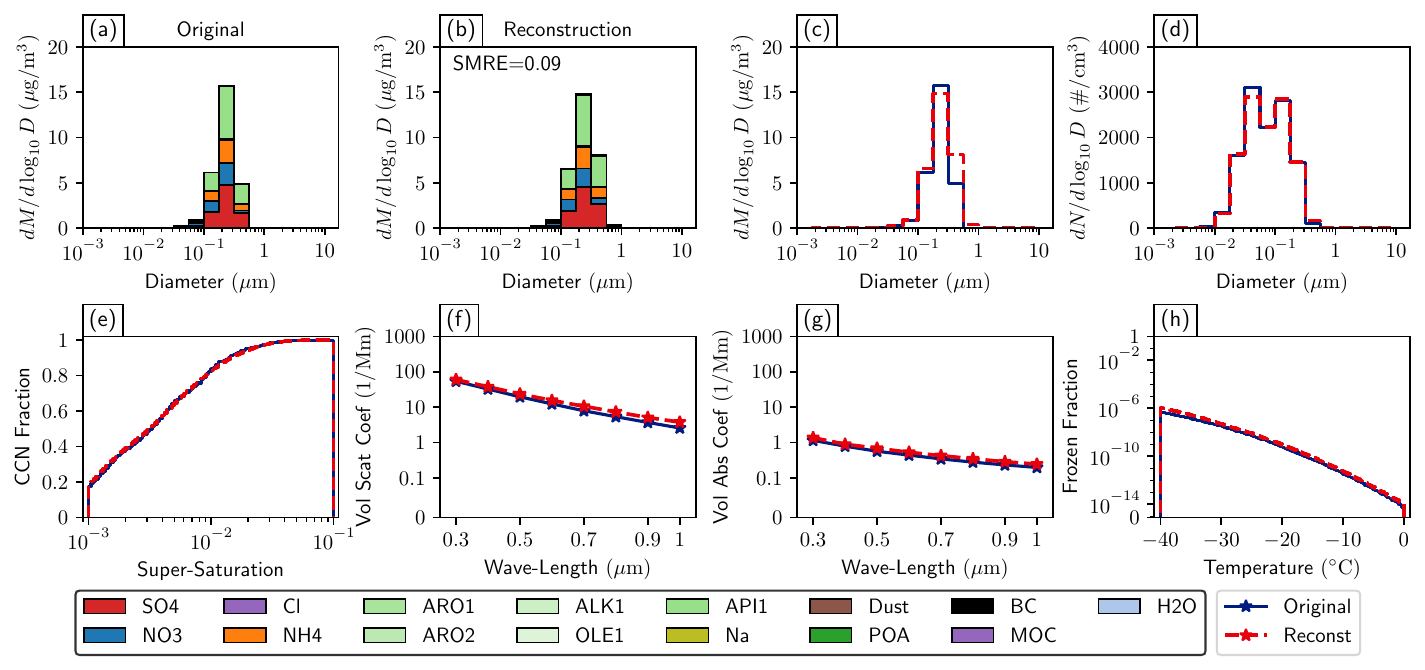}
	\includegraphics[page=8,width=0.8\linewidth]{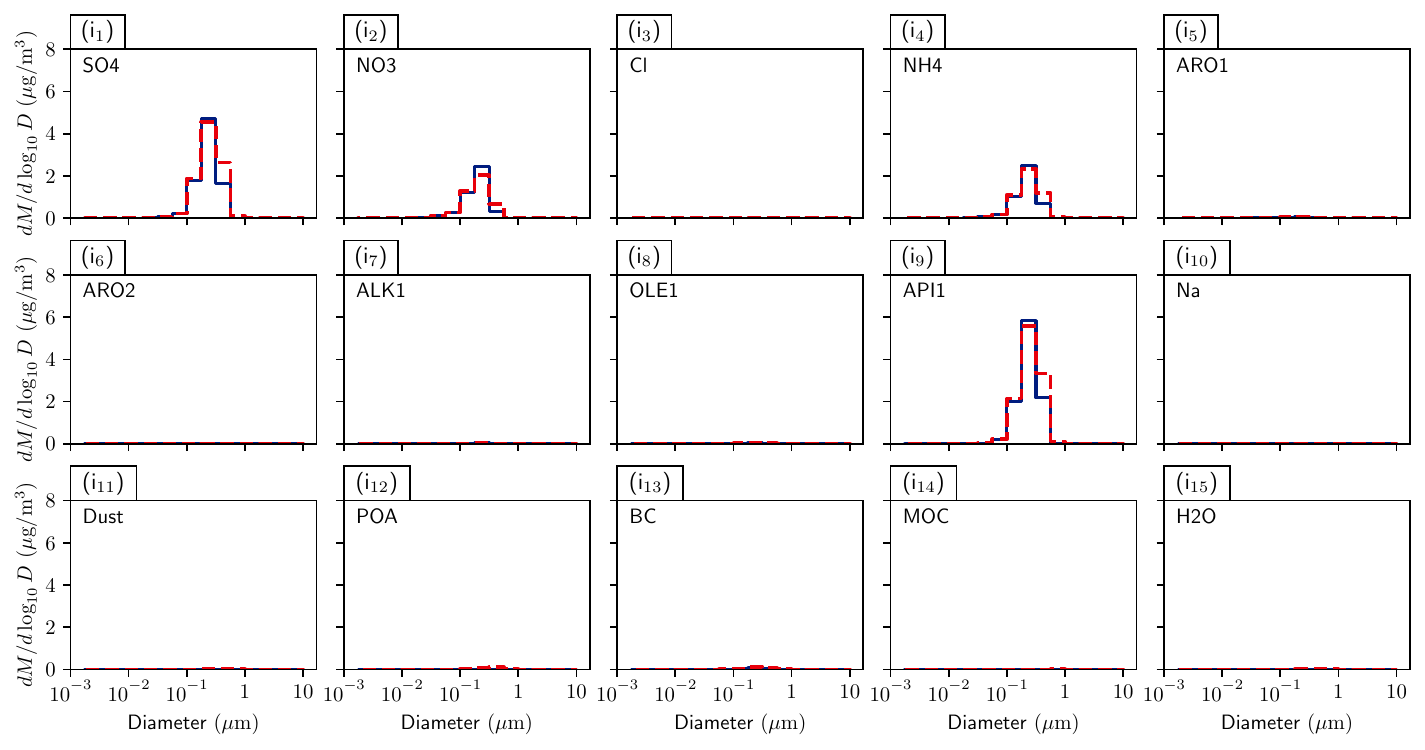}
	\includegraphics[page=9,width=0.82\linewidth]{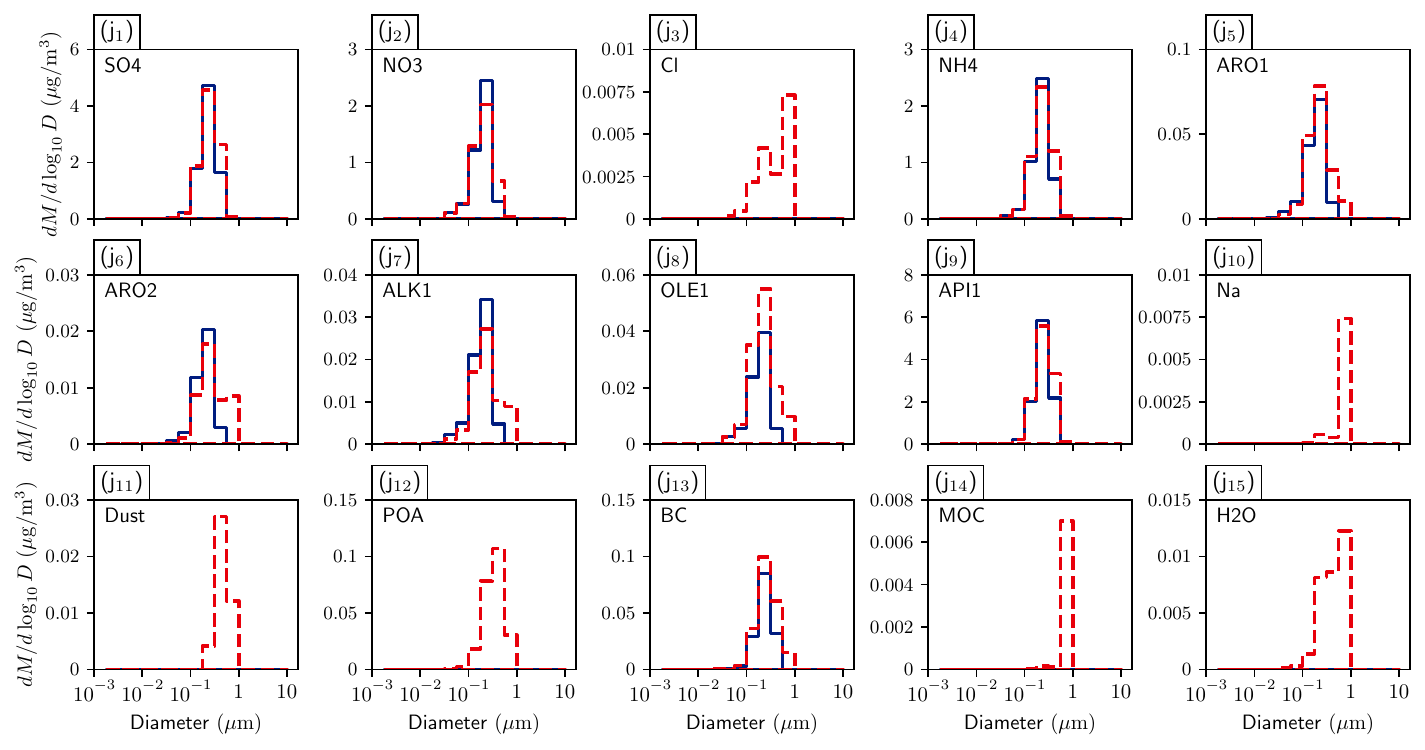}
	\caption{The aerosol diagnostics for another test sample with the same layout as Figure~\ref{fig:08anecdiag}. The particular sample in this figure has a speciated mass relative error of $0.09$.}\label{fig:a05anecdiag}
\end{figure}

\begin{figure}
	\centering
	\includegraphics[page=25,width=0.8\linewidth]{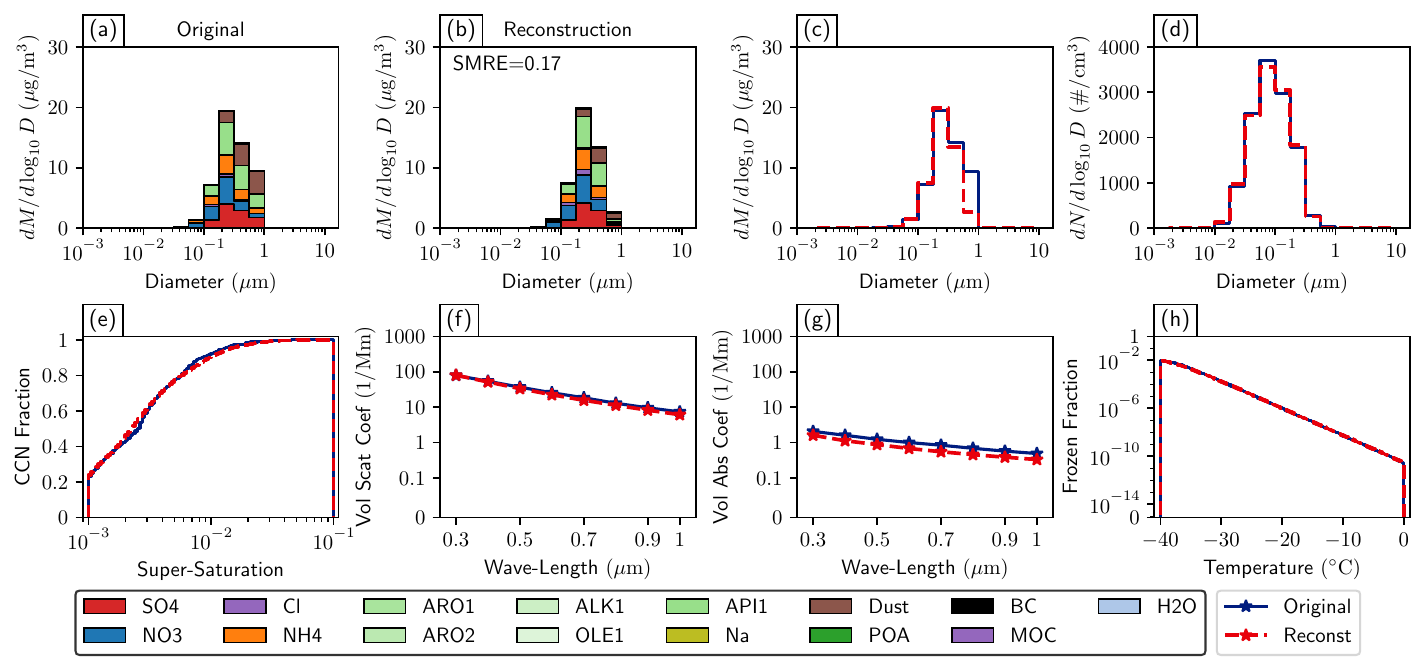}
	\includegraphics[page=26,width=0.8\linewidth]{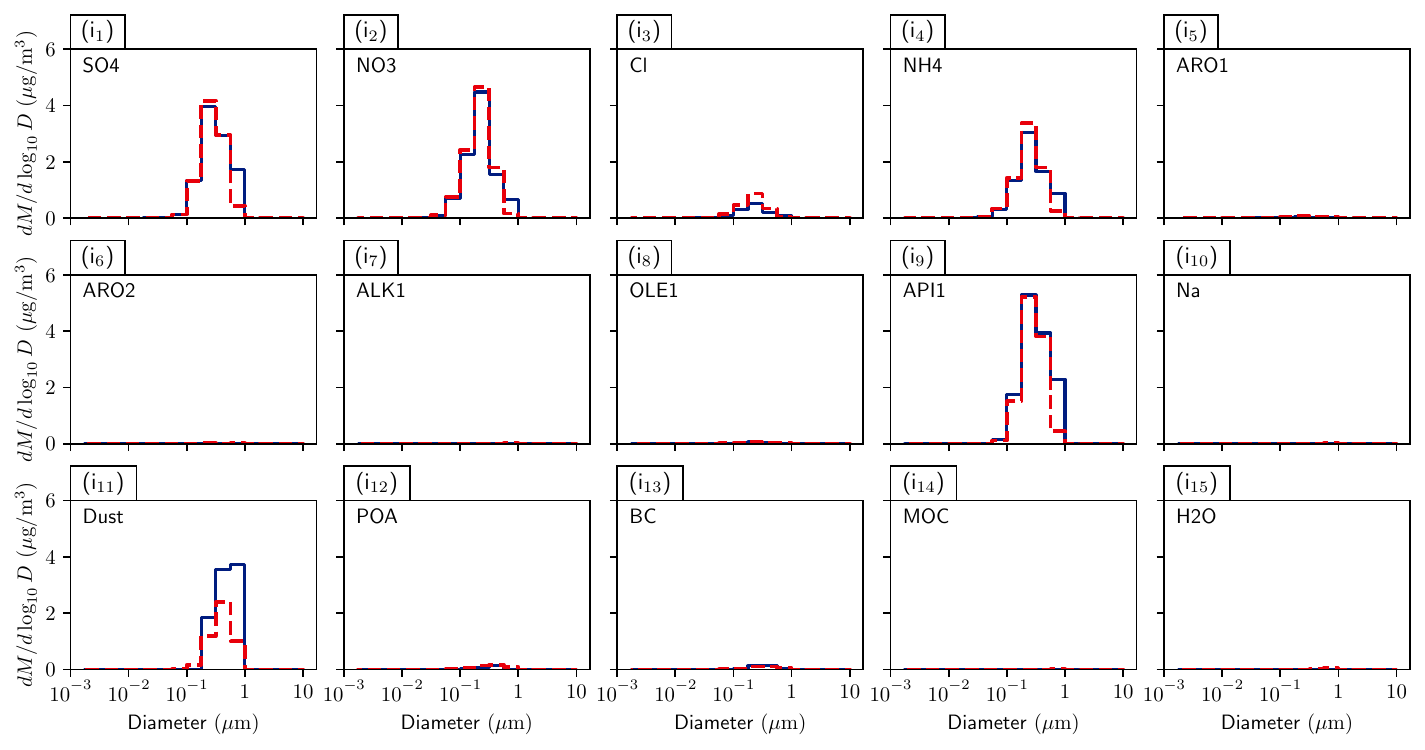}
	\includegraphics[page=27,width=0.82\linewidth]{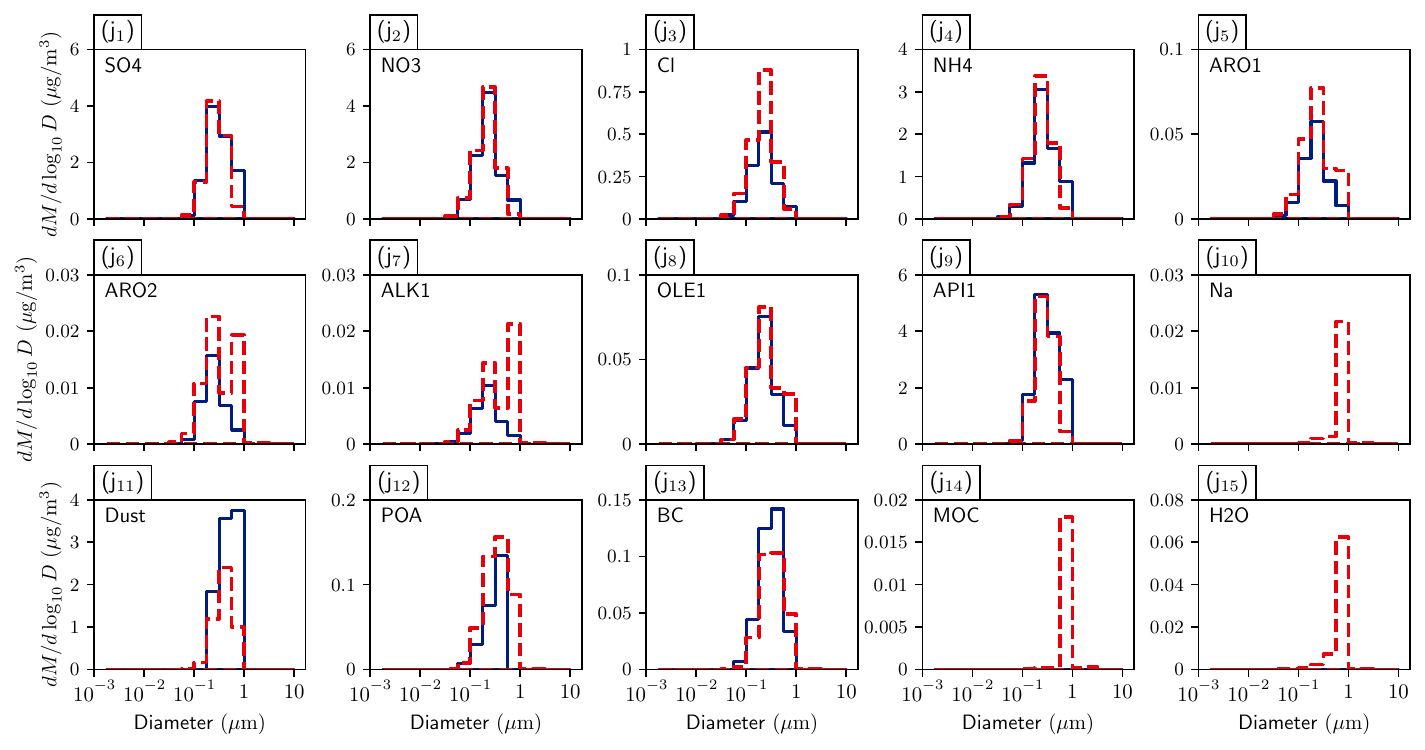}
	\caption{The aerosol diagnostics for another test sample with the same layout as Figure~\ref{fig:08anecdiag}. The particular sample in this figure has a speciated mass relative error of $0.17$.}\label{fig:a06anecdiag}
\end{figure}

\section{Implementation Details}

Models were trained on the particle-resolved aerosol data using an 80--20 train--test split over simulation scenarios, with all time points retained within each selected scenario, and results were aggregated over 10 random splits. The latent state had dimension $L=10$, consisting of a one-dimensional scale coordinate for total number concentration and a nine-dimensional latent shape coordinate. Since each sampled input population contains $1000 \times (15 + 1) = 16000$ scalar particle-state entries, this corresponds to a nominal 1600-fold reduction in scalar state dimension for the sampled representation used here. This particle count fixes the tensor shape and nominal compression ratio in the present experiments, but the encoder is defined as a weighted aggregation and can be evaluated on different numbers of particles. The encoder applied a two-hidden-layer MLP with 256 hidden units to each normalized particle composition vector, multiplied the resulting particle features by the normalized particle weights, and summed them to obtain the mean and diagonal covariance parameters for the latent shape distribution. The total number concentration was deterministically encoded as the scale coordinate. The decoder used a two-hidden-layer MLP with 256 hidden units to map the latent shape coordinate to the transformed aerosol diagnostics, while passing the scale coordinate through separately. We used ReLU activations and no normalization layers.

We trained with the Adam~\citep{kingma2014adam} optimizer using a learning rate of $10^{-3}$, a mini-batch size of 16, and 100000 training iterations. The reconstruction loss was the mean squared error in the transformed diagnostic space described in Section~\ref{sec:preproc}; the base model used a mean-reduced KL weight of 0.3 for both the latent mean and covariance terms and a mixup loss weight of 1.0, with interpolation coefficients sampled from $\mathrm{Beta}(1,1)$. For the diagnostic calculations, the CCN spectrum used 100 logarithmically spaced supersaturation values from $10^{-3}$ to $10^{-1}$ at $T=287\,\mathrm{K}$, with 150 Newton iterations for the critical-diameter calculation. Optical scattering and absorption were evaluated at wavelengths from 300 to 1000 nm in 100 nm increments, using 220 terms in the Toon--Ackerman Mie-series calculation. Ice-nucleation spectra were evaluated at 100 temperatures from $-40^\circ\mathrm{C}$ to $0^\circ\mathrm{C}$.

\section{Proof of Linearity}
\label{sec:direct-proof-linearity}

Here we explicitly prove the linearity of AeroMELD encoders given by (\ref{eqn:ecoder-form-scale})--(\ref{eqn:encoder-form-shape}). To prove that $\Phi$ is linear, we first observe that $\pi_3 = \alpha \pi_1 + \beta \pi_2$ has encoding $\Phi(\pi_3) = (n_3, z_3)$ where $n_3 = \alpha n_1 + \beta n_2$ and
\begin{align}
    z_3
    &= \sum_{i=1}^{N_1} \frac{\alpha n_{1,i}}{n_3} \phi(\mu_{1,i}) + \sum_{i=1}^{N_2} \frac{\beta n_{2,i}}{n_3} \phi(\mu_{2,i}) \\
    &= \frac{\alpha n_1}{\alpha n_1 + \beta n_2} \sum_{i=1}^{N_1} \frac{n_{1,i}}{n_1} \phi(\mu_{1,i}) + \frac{\beta n_2}{\alpha n_1 + \beta n_2} \sum_{i=1}^{N_2} \frac{n_{2,i}}{n_2} \phi(\mu_{2,i}) \\
    &= \frac{\alpha n_1}{\alpha n_1 + \beta n_2} z_1 + \frac{\beta n_2}{\alpha n_1 + \beta n_2} z_2.
\end{align}
This means that
\begin{align}
    \Phi(\pi_3)
    &= \left(\alpha n_1 + \beta n_2, \frac{\alpha n_1}{\alpha n_1 + \beta n_2} z_1 + \frac{\beta n_2}{\alpha n_1 + \beta n_2} z_2\right) \\
    &= \alpha (n_1, z_1) + \beta (n_2, z_2) \\
    &= \alpha \Phi(\pi_1) + \beta \Phi(\pi_2).
\end{align}

\section{Double-Scale Linear Operation Derivation}
\label{sec:double-scale-derivation}

Recall that the latent representation in AeroMELD is given by the pair $(n, z)$, where $n$ is the total number concentration and $z$ is the shape component of the latent variable. An equivalent representation is given by the double-scale pair $(n, s)$ where $s = nz$ is the scaled shape component. The linear operation in the latent representation is given by
\begin{equation}
\label{eqn:appendix-latent-linear}
\alpha (n_1, z_1) + \beta (n_2, z_2)
= \bigl(\alpha n_1 + \beta n_2, \lambda z_1 + (1 - \lambda) z_2\bigr), \text{ where } \lambda = \frac{\alpha n_1}{\alpha n_1 + \beta n_2},
\end{equation}
while the linear operation in the double-scale representation is given by
\begin{equation}
\label{eqn:appendix-double-scale-linear}
\alpha (n_1, s_1) + \beta (n_2, s_2)
= (\alpha n_1 + \beta n_2, \alpha s_1 + \beta s_2).
\end{equation}
Here we prove that the operation (\ref{eqn:appendix-latent-linear}) is equivalent to the double-scale operation given in (\ref{eqn:appendix-double-scale-linear}). Recall that $s = nz$, so $z = s/n$. Starting from the right-hand side of (\ref{eqn:appendix-latent-linear}), we have
\begin{align}
	\bigl(\alpha n_1 + \beta n_2, \lambda z_1 + (1 - \lambda) z_2\bigr)
	&= \left(\alpha n_1 + \beta n_2, \frac{\alpha n_1}{\alpha n_1 + \beta n_2} z_1 + \frac{\beta n_2}{\alpha n_1 + \beta n_2} z_2\right) \\
	&= \left(\alpha n_1 + \beta n_2, \frac{\alpha n_1}{\alpha n_1 + \beta n_2} \frac{s_1}{n_1} + \frac{\beta n_2}{\alpha n_1 + \beta n_2} \frac{s_2}{n_2}\right) \\
	&= \left(\alpha n_1 + \beta n_2, \frac{\alpha s_1 + \beta s_2}{\alpha n_1 + \beta n_2} \right).
\end{align}
Converting back to the double-scale representation using $s = nz$, the second component becomes
\begin{equation}
	(\alpha n_1 + \beta n_2) \cdot \frac{\alpha s_1 + \beta s_2}{\alpha n_1 + \beta n_2} = \alpha s_1 + \beta s_2,
\end{equation}
which gives the result in (\ref{eqn:appendix-double-scale-linear}).

%\clearpage
%%%%%%%%%%%%%%%%%%%%%%%%%%%%%%%%%%%%%%%%%%%%%%%
% References
%%%%%%%%%%%%%%%%%%%%%%%%%%%%%%%%%%%%%%%%%%%%%%%

\end{document}